\documentclass[twoside]{article}

\usepackage[accepted]{aistats2025}
\usepackage{amssymb}
\usepackage{float}
\usepackage{tikz}
\usetikzlibrary{bayesnet}
\usepackage{algorithm}
\usepackage{algpseudocode}
\usepackage{amsmath}
\usepackage{physics}
\usepackage{bbm}
\usepackage{enumitem}

\usepackage{pdflscape} % Add this to your preamble

\usepackage{tablefootnote}

\usepackage{siunitx}
\usepackage{dirtytalk}

\usepackage{graphicx}
\usepackage{caption}
\usepackage{subcaption}

\usepackage{pdflscape}       % For landscape orientation
\usepackage{tabularx}        % For flexible table widths
\usepackage{makecell}        % For better cell formatting
\usepackage{booktabs}        % For professional table lines
\usepackage{threeparttable}  % For table notes
\usepackage{array}           % For advanced table formatting
\usepackage{amsmath}         % For mathematical expressions
\usepackage{amssymb}         % For additional math symbols

\newcolumntype{L}{>{\raggedright\arraybackslash}p{5cm}} % Left-aligned fixed width
\newcolumntype{F}{>{\centering\arraybackslash}p{5cm}}   % Centered fixed width
\newcolumntype{C}[1]{>{\centering\arraybackslash}p{#1}}

\usepackage{url}
\DeclareRobustCommand{\parhead}[1]{\noindent \textbf{#1} }

\usepackage{subcaption}

\usepackage{amsthm}

\usepackage{listings}
\newcommand{\NA}{\multicolumn{1}{c}{\text{--}}}

\newtheorem{theorem}{Theorem}
\newtheorem{lemma}{Lemma}

\usepackage[accepted]{aistats2025}
\usepackage[round]{natbib}

\newcommand\blfootnote[1]{%
  \begingroup
  \renewcommand\thefootnote{}\footnote{#1}%
  \addtocounter{footnote}{-1}%
  \endgroup
}

\begin{document}



\twocolumn[

\aistatstitle{Posterior Mean Matching: Generative Modeling through Online Bayesian Inference}

\aistatsauthor{ Sebastian Salazar$^{1,3}$ \And Michal Kucer$^{1}$ \And Yixin Wang$^{2}$
\And Emily Casleton$^{1}$ \And David Blei$^{3}$}

 \aistatsaddress{ $^{1}$Los Alamos National Laboratory \And  $^{2}$University of Michigan \And $^{3}$Columbia University } ]

\begin{abstract}
This\blfootnote{For further correspondence, please email \{ssalazar, michal\}@lanl.gov.} paper introduces posterior mean matching (PMM), a new method for generative modeling that is grounded in Bayesian inference. PMM uses conjugate pairs of distributions to model complex data of various modalities like images and text, offering a flexible alternative to existing methods like diffusion models. PMM models iteratively refine noisy approximations of the target distribution using updates from online Bayesian inference. PMM is flexible because its mechanics are based on general Bayesian models. We demonstrate this flexibility by developing specialized examples: a generative PMM model of real-valued data using the Normal-Normal model, a generative PMM model of count data using a Gamma-Poisson model, and a generative PMM model of discrete data using a Dirichlet-Categorical model. For the Normal-Normal PMM model, we establish a direct connection to diffusion models by showing that its continuous-time formulation converges to a stochastic differential equation (SDE). Additionally, for the Gamma-Poisson PMM, we derive a novel SDE driven by a Cox process, which is a significant departure from traditional Brownian motion-based generative models.  PMMs achieve performance that is competitive with generative models for language modeling and image generation.
\end{abstract}

\section{Introduction}
The goal of generative modeling is to use data $\{\pmb{x}_i\}_{i=1}^{n}$ to produce new samples from a target distribution $p^\star(\pmb{x})$. The challenge is that $\pmb{x}$
is high dimensional and $p^\star(\pmb{x})$ is complex \citep{mackay2003information}.

Here are some examples:
\begin{itemize}[leftmargin=*,itemsep=0pt]
    \item The data are natural images; the target is the distribution of images found in the world; the goal is to produce realistic images \citep{ho2020denoising, goodfellow2014generative}.

    \item  The data are documents; the target is the distribution of fluent language; the goal is to produce coherent text \citep{vaswani2017attention}.

    \item The data are gene sequences of proteins; the target is the distribution of stable proteins; the goal is to produce new proteins with specific properties \citep{protein}.

    \item Probabilistic prediction in tabular data, where the goal is to model the conditional distribution of a response variable given a collection of features \citep{beltranvelez2024treeffuserprobabilisticpredictionsconditional, salazar2024vart}.
\end{itemize}

In this paper, we develop \textit{posterior mean matching} (PMM), a new method of generative modeling that is flexible enough to solve all of these problems. The key property of PMM is that it is based on the machinery of online Bayesian inference. It inherits the flexibility of Bayesian modeling, and so it is easy to apply to many types of data and target distributions.

To develop PMM, we first posit a conjugate Bayesian model and show how, in theory, it can be used to sample exactly from the target $p^*(\pmb{x})$. We then show how to use variational inference variational inference to approximate a distribution that produces such exact samples. PMM is flexible because it can employ any conjugate Bayesian model in an inner routine.% TODO: 

We study PMM on images and text. For image generation, we develop a PMM method based on an underlying Gaussian/Gaussian model. We find that it produces Frechet inception distance (FID) scores \citep{fidheusel2017gans} that are comparable to most diffusion models \citep{Karras2022edm, dhariwal2021diffusionmodelsbeatgans, song2020score}. To apply PMM to text, we simply swap the Gaussian model for a Dirichlet/Categorical. We find that the text-generating PMM models offer competitive performance to diffusion non-autoregressive language models \citep{sedd, mdlm, shi2024simplified}. 

\parhead{Related Work.} Generative modeling is an active area of machine learning research. For images, some popular methods include variational autoencoders  \citep{kingma2013auto, rezende2014stochastic}, generative adversarial networks \citep{goodfellow2014generative}, normalizing flows \citep{dinh2014nice, rezende2015variational}, autoregressive models \citep{van2016conditional}, and diffusion models \citep{ho2020denoising}. For text, the main method is the transformer-based autoregressive models \citep{vaswani2017attention}. PMM is a contribution to this research area, providing an easily adaptable method for generative modeling, applicable to text, images, and many other types of data. 
While on images PMM compares favorably to diffusions, on text, its performance is competitive with other non-autoregressive diffusion-based language models \citep{d3pm, sedd, mdlm, shi2024simplified}.
PMMs are also related to diffusion models, we establish this technical connection in Section \ref{sec:sde}.

Closest in spirit to PMMs is Bayesian flow nets \citep{graves2024bayesianflownetworks} (BFNs), which also use Bayesian methods in the context of generative modeling. PMM is based on exact sampling from the target, while BFNs are motivated by information theoretic principles. PMM provides a simpler algorithm than BFN, and performed better in our studies of text data in Section \ref{sec:exper}.

\parhead{Contribution.} Posterior mean matching (PMM) contributes to the field of generative modeling by offering a unified and adaptable method grounded in  Bayesian inference. PMM easily applies to diverse data types such as images, text, and count data. 

\section{Posterior Mean Matching} 
There are several ingredients to posterior mean matching. Throughout, we assume that we are given a dataset  \( \{ \pmb{x}_1, \ldots, \pmb{x}_n \} \) of i.i.d. samples from the target distribution $p^{*}(\pmb{x})$. 

\parhead{Noisy observation model.} The first ingredient is the \textit{noisy observation model}. It is a conditional distribution $\pi_{\alpha}(\pmb{y} | \pmb{x})$ that is easy to sample from (e.g., a Gaussian, Poisson, Categorical). Samples from this conditional are interpreted as noisy versions of $\pmb{x}$. 

\parhead{Augmented Target Distribution.} We augment the target distribution \( p^{*}(\pmb{x}) \) with the noisy observation model \( \pi_{\alpha_{s}}(\pmb{y} \mid \pmb{x}) \) and define a joint distribution over \( (\pmb{x}, \pmb{y}_{1:t}) \), termed the \textit{augmented target distribution}:
\begin{align}
    \pmb{x} &\sim p^*(\pmb{x}), \label{eqn:target_prior} \\ 
    y_{s} \mid \pmb{x} &\overset{\perp}{\sim} \pi_{\alpha_{s}}(\pmb{y} \mid \pmb{x}), \quad s = 1, \ldots, t, \\ 
    p(\pmb{x}, \pmb{y}_{1:t}) &\equiv p^{*}(\pmb{x}) \prod_{s=1}^t \pi_{\alpha_{s}}(\pmb{y}_s \mid \pmb{x}). \label{eqn:joint_truth}
\end{align} 

Where we have introduced a sequence of hyperparameters $\alpha_{1}, ..., \alpha_{t}$, where $\alpha_{s}$ can be interpreted as a parameter modulating the amount of noise in the sample $\pmb{y}_{s}$ (e.g., the precision parameter of a Normal distribution). The augmented model simply augments draws \( \pmb{x}^{*} \) from the target \( p^*(\pmb{x}) \) with a collection of noisy observations \( \pmb{y}_{1}, \ldots, \pmb{y}_{t} \). Note the ``prior'' distribution (\ref{eqn:target_prior}) of this generative process is the target distribution \( p^*(\pmb{x}) \), which is not directly available. 

\parhead{Augmented Bayesian Model.} The next ingredient is the \textit{augmented Bayesian model}. This model is identical to the \textit{augmented target distribution} except that the unknown target \( p^{*}(\pmb{x}) \) is replaced with a known distribution \( \pi(\pmb{x}) \), that serves as a known ``prior.'' The augmented Bayesian model is
\begin{align}
    \pmb{x} &\sim \pi(\pmb{x}), \label{eqn:prior} \\ 
    \pmb{y}_s \mid \pmb{x} &\overset{\perp}{\sim} \pi_{\alpha_s}(\pmb{y} \mid \pmb{x}), \quad s = 1, \ldots, t \\ 
    \pi(\pmb{x}, \pmb{y}_{1:t}) &\equiv \pi(\pmb{x}) \prod_{s=1}^t \pi_{\alpha_s}(\pmb{y}_s \mid \pmb{x}). \label{eqn:bayesian_model}
\end{align} 

We require the augmented Bayesian model to satisfy the following three properties. 

First, the posterior expectation $\mathbb{E}_{\pi}(\pmb{x} | \pmb{y}_{1:t})$ must have a known closed form. This is facilitated by picking a prior $\pi(\pmb{x})$ that is conjugate to the noise model $\pi_{\alpha_{s}}(\pmb{y} | \pmb{x})$, e.g., a normal prior with a normal noisy observation model. 

Second, the posterior mean must be \textit{consistent}. Given a collection of noisy samples $\pmb{y}_{s} \sim \pi_{\alpha_{s}}(\pmb{y} | \pmb{x}^{*})$ for $s = 1, ..., t$, we say that the posterior mean is consistent if it eventually recovers the \textit{true} $\pmb{x}^{*}$ that was used to generate the noisy samples $\pmb{y}_{1}, ..., \pmb{y}_{t}$, namely 
\begin{align}
    \pmb{\mu}_{t} \equiv \mathbb{E}_{\pi}(\pmb{x} | \pmb{y}_{1:t}) \overset{a.s.}{\to} \pmb{x}^{*}.
\end{align}
All of the PMM models considered in this paper are consistent (see Appendix for consistency proofs).

\begin{figure}[t]
    \centering
    \begin{tikzpicture}

        % Nodes
        \node[obs] (y1) {$\pmb{y}_1$};
        \node[obs, right=of y1] (y2) {$\pmb{y}_2$};
        \node[obs, right=of y2] (y3) {$\pmb{y}_3$};
        \node[obs, right=of y3] (dots) {$\cdots$};  % Ellipsis for continuation
        \node[det, below=of y1] (mu1) {$\pmb{\mu}_1$};
        \node[det, below=of y2] (mu2) {$\pmb{\mu}_2$};
        \node[det, below=of y3] (mu3) {$\pmb{\mu}_3$};
        \node[det, below=of dots] (mu_dots) {$\cdots$};  % Ellipsis for continuation
        \node[det, left=of mu1] (mu0) {$\pmb{\mu}_0$};

        % Edges
        \edge {mu0} {mu1}; 
        \edge {mu1} {mu2};
        \edge {mu2} {mu3};
        \edge {mu3} {mu_dots};  % Edge to ellipsis
        \edge {y1} {mu1};
        \edge {y2} {mu2};
        \edge {y3} {mu3};
        \edge {dots} {mu_dots};  % Edge from observation ellipsis to deterministic ellipsis

    \end{tikzpicture}
    \caption{Diagram of the online Bayesian inference update process. At each time step \( t \), an observation \( \pmb{y}_t \) is incorporated to update the posterior mean \( \pmb{\mu}_t \). The ellipsis (\(\cdots\)) indicates the iterative nature of the updates, starting from the prior mean \( \pmb{\mu}_{0} \).}
    \label{fig:bayesian_update}
\end{figure}
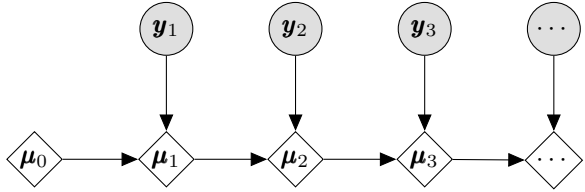

Finally, the augmented model must be amenable to online Bayesian inference. This means that it is possible to write an update rule for the posterior mean 
\begin{align}
    \pmb{\mu}_{t+1} &= f_{t}(\pmb{\mu}_{t}, \pmb{y}_{t+1}). \label{eqn:online_bayes}
\end{align}
Figure \ref{fig:bayesian_update} diagrams online Bayesian inference.

\begin{figure}[t]
    \centering
    \includegraphics[width=0.98\columnwidth]{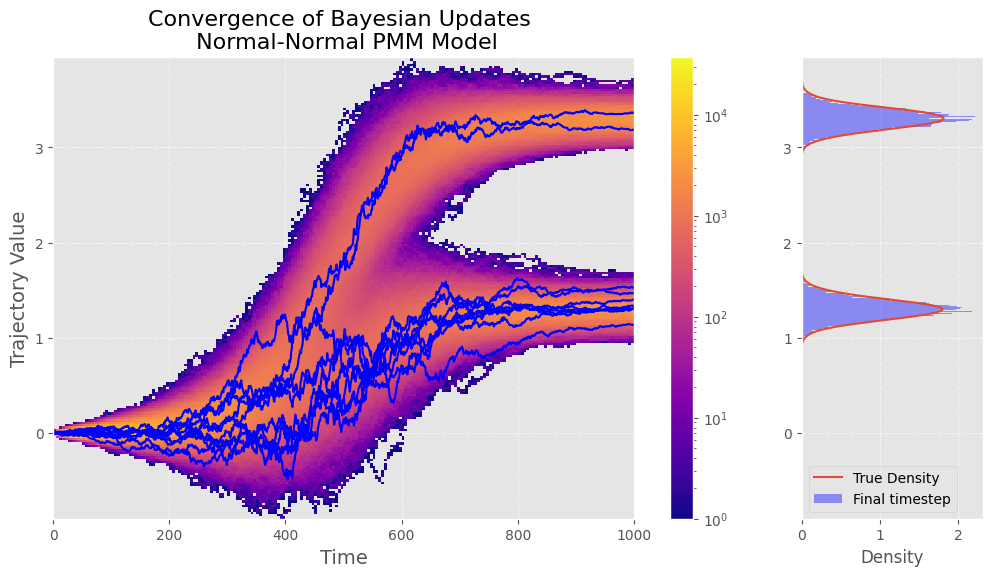}
    \caption{Convergence of the posterior mean trajectories \( \pmb{\mu}_t \) to samples from the target \( \pmb{x} \sim p^*(\pmb{x}) \) as \( t \) increases for the Normal Posterior Mean Matching (PMM) model. Refer to Figure \ref{fig:pmean_convergence_app} in the Appendix \ref{app:figures} for a more detailed view.}
    \label{fig:normal_pmm}
\end{figure}

\parhead{Generative Modeling with online Bayesian Inference.} With these ingredients---the augmented target and the augmented model---we show how to use augmented data from (\ref{eqn:joint_truth}) and online Bayesian inference from the augmented Bayesian model (\ref{eqn:bayesian_model}) to produce a neural-network-based sampler from the target distribution $p^{*}(\pmb{x})$.

We start by considering data from the augmented target distribution $\{ \pmb{x}^{*}, \pmb{y}_{1}, ..., \pmb{y}_{t}\}$, which we generate by taking a sample $\pmb{x}^{*}$ from the target distribution $p^{*}(\pmb{x})$ and then producing a sequence of $t$ noisy observations $\pmb{y}_{1:t}$ using the noise model $\pi_{\alpha_{t}}(\pmb{y} | \pmb{x})$. In practice, we approximate the target distribution $p^{*}(\pmb{x})$ by taking a random sample $\pmb{x}_{i}$ from our dataset $\{\pmb{x}_{1}, ..., \pmb{x}_{n}\} $ of i.i.d. samples from $p^{*}(\pmb{x})$. 

Using the augmented sample $\{ \pmb{x}^{*}, \pmb{y}_{1}, ..., \pmb{y}_{t}\}$, we consider the sequence of posterior means $\pmb{\mu}_{1}, ..., \pmb{\mu}_{t}$ with respect to the augmented Bayesian model; the sequence of posterior means has the following properties:
\begin{itemize}[leftmargin=*,itemsep=0pt]
    \item $\pmb{\mu}_{1}, ..., \pmb{\mu}_{t}$ is a sequence of random variables. Their randomness is inherited from the noisy observation model $\pi_{\alpha_{t}}(\pmb{y} | \pmb{x})$. 
    \item For all $s \in \{0, ..., t-1\} $, it is easy to calculate $\pmb{\mu}_{s+1}$ from $\pmb{\mu}_{s}$ and $\pmb{y}_{s+1}$ using online Bayesian inference. This was one of the requirements of the augmented Bayesian model.
    \item The limit $\lim_{t \to \infty} \pmb{\mu}_{t}$ converges to $\pmb{x}^{*}$ --- a sample from the target distribution $p^{*}(\pmb{x})$. This is a consequence of the consistency of the posterior expectation, another requirement of the augmented Bayesian model.
\end{itemize}
These three properties suggest a strategy to draw samples from the target $p^{*}(\pmb{x})$. 
\begin{enumerate}[leftmargin=*,itemsep=0pt]
    \item Obtain a sample $\{\pmb{x}^{*}, \pmb{y}_{1}, ..., \pmb{y}_{t}\}$ and throw away $\pmb{x}^{*}$. This results in a sequence $\pmb{y}_{1}, ..., \pmb{y}_{t}$, that is viewed as a sample from the marginal distribution of the augmented target (\ref{eqn:joint_truth}). 
    \item Using the augmented Bayesian model, compute the sequence of posterior means $\pmb{\mu}_{1}, ..., \pmb{\mu}_{t}$ using online Bayesian Inference. It is worth highlighting that the expectation is defined using the augmented Bayesian model of equation (\ref{eqn:bayesian_model}), while the data  \( \pmb{y}_{1:t} \) used to compute this expectation, are random variables drawn from the marginal distribution of the augmented target (\ref{eqn:joint_truth}). 
    \item Because this sequence is consistent, for $t$ large enough $\pmb{\mu}_{t} \approx \pmb{x}^{*}$. In other words, the posterior mean $\pmb{\mu}_{t}$ is effectively a sample from the target distribution $p^{*}(\pmb{x})$. 
\end{enumerate}
This logic implies that sampling from the target distribution $p^{*}(\pmb{x})$ reduces to sampling $\pmb{\mu}_{t}$ from the joint distribution of posterior means $p(\pmb{\mu}_{1}, ..., \pmb{\mu}_{t})$. 

We illustrate sample trajectories $\pmb{\mu}_{1}, ..., \pmb{\mu}_{t}$ in Figure \ref{fig:normal_pmm}, where the target is a bimodal distribution.  We can see that $\pmb{\mu}_{t}$ converges to samples from the target $p^{*}(\pmb{x})$. 

\parhead{Approximately Sampling from the Target.} In practice, the joint distribution \(p(\pmb{\mu}_{1}, \ldots, \pmb{\mu}_{t})\) is intractable to sample from exactly. So, we sample from the target by approximating the joint distribution of posterior means, and taking samples of $\pmb{\mu}_t$. 

We approximate $p(\pmb{\mu}_{1}, \ldots, \pmb{\mu}_{t})$ by introducing a family of distributions \( q_{\pmb{\varphi}}(\pmb{\mu}_{1}, \ldots, \pmb{\mu}_{t}) \) and minimizing the following objective function:
\begin{align}
    \mathcal{L}_{\text{PMM}}(\pmb{\varphi}) &= \text{KL}(p(\pmb{\mu}_{1}, \ldots, \pmb{\mu}_{t}) \| q_{\pmb{\varphi}}(\pmb{\mu}_{1}, \ldots, \pmb{\mu}_{t})). \label{eqn:pmm_obj}
\end{align}

The form of \( q_{\pmb{\varphi}}(\pmb{\mu}_{1:t}) \) is motivated by mechanics of online Bayesian inference and is defined implicitly:
\begin{align}
    \pmb{\hat{y}}_{t+1} &\sim \pi_{\alpha_{t+1}}(\pmb{y} \mid g_{\pmb{\varphi}}(\pmb{\mu}_{t}, t))  \label{eqn:noisy_model_online_approx}\\ 
    \pmb{\mu}_{t+1} &= f_{t}(\pmb{\mu}_{t}, \pmb{\hat{y}}_{t+1}). \label{eqn:online_bayes_approx}
\end{align}
Here, \( g_{\pmb{\varphi}} \) is a flexible function parameterized with a neural network. Given data, we learn the neural network \( g_{\pmb{\varphi}} \) by minimizing the PMM objective in Equation (\ref{eqn:pmm_obj}). Once fit, we can obtain approximate samples from the target distribution \( p^{*}(\pmb{x}) \) by iteratively applying equations (\ref{eqn:noisy_model_online_approx}) and (\ref{eqn:online_bayes_approx}). This sampling procedure is detailed in Algorithm \ref{alg:pmm_inference}.
\begin{algorithm}[t]
\caption{Sampling \( p^{*}(\pmb{x}) \) from a fitted PMM}
\begin{algorithmic}[1]
    \State \textbf{Initialize:} Set \( \pmb{\mu}_{0} \) to the prior mean.
    \For{\( s = 1 \) \textbf{to} \( t \)}
        \State Compute \( \hat{\pmb{x}}_{s} \gets g_{\pmb{\varphi}}(\pmb{\mu}_{s-1}, s) \)
        \State Sample \( \pmb{\hat{y}}_{s} \sim \pi_{\alpha_{s}}(\pmb{y} \mid \hat{\pmb{x}}_{s}) \)
        \State Update \( \pmb{\mu}_{s} \) using the online Bayesian Inference update rule \( \pmb{\mu}_{s} = f_{s}(\pmb{\mu}_{s-1}, \pmb{\hat{y}}_{s}) \)
    \EndFor
    \State \textbf{Output:} Return \( \pmb{\mu}_{t} \)
\end{algorithmic}
\label{alg:pmm_inference}
\end{algorithm}

\section{Examples of PMM Models}
We now work out the components of posterior mean matching using three conjugate pairs of distributions: Normal-Normal, Gamma-Poisson, and Dirichlet-Categorical models. These models are suitable for real-valued, positive, and text data, respectively. 

\subsection{Normal-Normal PMM: a generative model of real-valued data}
\parhead{Data Representation.} This section concerns the Normal-Normal PMM, a generative model designed to model real-valued data. This boils down to assuming that samples from the target distribution $p^{*}(\pmb{x})$ are vectors in $\mathbb{R}^{d}$. 

\parhead{Augmented Target Distribution.} The Normal-Normal PMM posits a \emph{noisy observation model} that corrupts samples $\pmb{x}^{*}$ from the target distribution $p^{*}(\pmb{x})$ through additive Gaussian noise $\pmb{y}_{t} \sim \mathcal{N}(\pmb{x}^{*}, \alpha_{t}^{-1}I)$. This noisy observation model defines the following augmented target distribution
\begin{align}
    p(\pmb{x}, \pmb{y}_{1:t}) &\equiv p^{*}(\pmb{x}) \prod_{s} \mathcal{N}(\pmb{y}_{s} ; \pmb{x}, \alpha_{s}^{-1}I) \label{eqn:norm-norm-pmm} \
\end{align}
In this context, the precision parameter $\alpha_{t}$ modulates the level of corruption in the noisy observations. 

\parhead{Augmented Bayesian Model.} Suppose we are given a sample $\pmb{y}_{1:t}$ from the marginal distribution of the augmented target 
\begin{align}
    p(\pmb{y}_{1:t}) &= \int p^{*}(\pmb{x}) \prod_{s} \mathcal{N}(\pmb{y}_{s} ; \pmb{x}, \alpha_{s}^{-1}I) d \pmb{x}. \label{eqn:norm-norm-pmm-marginal}
\end{align} 
Based on equations (\ref{eqn:norm-norm-pmm}) and (\ref{eqn:norm-norm-pmm-marginal}) we know that there exists an $\pmb{x}^{*}$, that is a sample from $p^{*}(\pmb{x})$ such that $\pmb{y}_{s} \sim \mathcal{N}(\pmb{x}^{*}, \alpha_{s}^{-1} I)$. However, we only assume that $\pmb{y}_{1:t}$ is given to us---the mean parameter $\pmb{x}^{*}$ of these normal distributions is kept hidden. We infer $\pmb{x}^{*}$ by using a Normal-Normal augmented Bayesian model
\begin{align}
    \pmb{x} &\sim \mathcal{N}(\pmb{0}, \beta^{-1}I) \label{eqn:normal-prior} \\
    \pmb{y}_{s} | \pmb{x}, \alpha_{s}^{-1} &\sim \mathcal{N}( \pmb{x}, \alpha_{s}^{-1}I) \label{eqn:normal-likelihood}. 
\end{align}
\parhead{Online Bayesian Inference Update.} It is possible to calculate the posterior mean in the Normal-Normal model using the following update rule (see Appendix \ref{apx:normal-online})
\begin{align}
    \pmb{\mu}_{t} | \pmb{\mu}_{t-1} ,\pmb{y}_{t} &= \frac{\beta + \sum_{s=1}^{t-1} \alpha_{s}}{\beta + \sum_{s=1}^{t}\alpha_{s}}\pmb{\mu}_{t-1} + \frac{\alpha_{t} \pmb{y}_{t}}{\beta + \sum_{s=1}^{t}\alpha_{s}}  \label{eqn:normal-online-bayes}
\end{align}
where $\pmb{y}_{t} \sim \mathcal{N}(\pmb{x}, \alpha_{t}^{-1}I)$. The following theorem rigorously establishes the convergence of this posterior mean to $\pmb{x}^{*}$. 
\begin{theorem} \label{thm:norm-norm}
(Concentration of posterior mean) Let \( \{ \pmb{y}_1, \ldots, \pmb{y}_t \} \) be observations generated according to equation (\ref{eqn:norm-norm-pmm-marginal}).  Suppose \( \alpha_t \) a known, positive, increasing sequence satisfying \( \lim_{t \to \infty} \alpha_t = \infty \). Then, the posterior mean \( \pmb{\mu}_t \) of  the Bayesian model in equations (\ref{eqn:normal-prior}) and (\ref{eqn:normal-likelihood}) is consistent, namely:
\begin{align}
    \lim_{t \to \infty} \pmb{\mu}_t = \pmb{x}, \quad \text{almost surely},
\end{align}
with respect to the joint distribution of \( (\pmb{x}, \pmb{y}_1, \pmb{y}_2, \ldots) \) in equation (\ref{eqn:norm-norm-pmm}). 
\end{theorem}
Theorem \ref{thm:norm-norm} establishes the correctness of the approximate sampling scheme shown in Algorithm \ref{alg:pmm_inference} for the Normal-Normal model, which implies that in the limit, the posterior mean of this Bayesian model is effectively a sample from the target $p^{*}(\pmb{x})$. A visual demonstration of Theorem \ref{thm:norm-norm} is shown in Figure \ref{fig:normal_pmm}.

Putting together all of these components we now show how to compute the PMM objective for this model. 

\parhead{Normal-Normal Posterior Mean Matching Objective.} Using the online Bayesian Inference update, we approximate the posterior mean updates of equation (\ref{eqn:normal-online-bayes}) using a Neural Network $g_{\pmb{\varphi}}$ as in equations (\ref{eqn:noisy_model_online_approx}) and (\ref{eqn:online_bayes_approx}) as follows 
\begin{align}
    \hat{\pmb{y}}_{s} &\sim \mathcal{N}( g_{\pmb{\varphi}}(\pmb{\mu}_{s-1}, s), \alpha_{s}^{-1}) \\
    \pmb{\mu}_{t} | \pmb{\mu}_{t-1} &= \frac{\beta + \sum_{s=1}^{t-1} \alpha_{s}}{\beta + \sum_{s=1}^{t}\alpha_{s}}\pmb{\mu}_{t-1} + \frac{\alpha_{t} \hat{\pmb{y}}_{t}}{\beta + \sum_{s=1}^{t}\alpha_{s}}  \label{eqn:approx-normal-online-bayes}
\end{align}
Substituting (\ref{eqn:normal-online-bayes}) and (\ref{eqn:approx-normal-online-bayes}) into the PMM objective we obtain (see appendix \ref{apx:normal-pmm-obj})
\begin{align}
    \mathcal{L}_{\text{PMM}}(\pmb{\varphi}) & \propto t \cdot \mathbb{E}_{\substack{ 
       s \sim U(\{1, \dots, t\}) \\ 
       \pmb{x} \sim p^{*}(\pmb{x}) \\ 
       \pmb{\mu}_{s-1} | \pmb{x} 
   }}
    \alpha_{s} \norm{\pmb{x} - g_{\pmb{\varphi}}(\pmb{\mu}_{s-1}, s)}_{2}^{2} \label{eqn:norm-norm-obj}
\end{align}
\subsection{Dirichlet-Categorical PMM: a generative model of text}
\parhead{Data Representation.} In this section, we develop a Dirichlet-Categorical PMM to model a collection of text documents. This boils down to assuming that samples from the target distribution $p^{*}(\pmb{X})$ come from a discrete, finite space. Specifically, we represent each document in a corpus as a sequence of tokens \( \mathbf{X}^{*} = (\mathbf{x}_1, \ldots, \mathbf{x}_C) \), where each token \( \mathbf{x}_c \in \{0,1\}^V \cap \Delta^{V-1} \) is one-hot encoded from a fixed vocabulary of size \( V \).

\parhead{Augmented Target Distribution.} For every document $\pmb{X} =  (\pmb{x}_1, ..., \pmb{x}_C)$, suppose we generate a sequence of noisy documents $\pmb{Y}_1, ..., \pmb{Y}_T$ according to the following generative process:\footnote{Where $\pmb{Y}_t = (\pmb{y}_{t1},..., \pmb{y}_{tC})$ with $\pmb{y}_{tc} \in \{0, 1\}^{V+1} \cap \Delta^{V} $} 
\begin{align}
    (\pmb{x}_1, ..., \pmb{x}_C)  &\sim p^{*} (\pmb{x}_1, ..., \pmb{x}_C), \label{eqn:text_data} \\
\pmb{y}_{tc} | \pmb{x}_c, \alpha_{tc}  &\sim   \text{Cat}_{V+1}\left( \alpha_{tc} x_{c}^{(1)}, ..., \alpha_{tc}x_{c}^{(V)}, 1-\alpha_{tc} \right). \label{eqn:text_gen_cat}
\end{align}
Here, the $V+1^{th}$ token of the categorical random variable in equation (\ref{eqn:text_gen_cat}) should be thought of as a $<$mask$>$ token, representing missing data in the noisy observations of a document, and $w_{tc} \in [0,1]$ represents the probability that a token at position $c$ is unmasked at time $t$. 

\parhead{Augmented Bayesian Model.} A sample from the marginal distribution induced by augmented target distribution of equations (\ref{eqn:text_data}) and (\ref{eqn:text_gen_cat}) consists of a sequence of noisy versions $\pmb{Y}_{1}, ..., \pmb{Y}_{t}$ of a document $\pmb{X}$. It is possible to infer $\pmb{X}$ with the marginal samples of this noisy observation model using the following augmented Bayesian model. 
\begin{align}
(\pmb{x}_1, ..., \pmb{x}_C) &\sim_{\text{iid}} \text{Dir}_{V}(1/K), \label{eqn:text_prior} \\
\pmb{y}_{tc} | \pmb{x}_c, \alpha_{tc}  &\sim   \text{Cat}_{V+1}\left( \alpha_{tc} x_{c}^{(1)},..., \alpha_{tc}x_{c}^{(V)}, 1-\alpha_{tc} \right).
\label{eqn:text_likelihood}
\end{align}
\parhead{Online Bayesian Inference: Encoding Prior Knowledge with a Non-Informative Prior.} We present the online Bayesian update rule assuming a non-informative Dirichlet Prior (i.e. taking $K \to \infty$ in (\ref{eqn:text_prior})). The purpose of this choice is twofold:
\begin{enumerate}[leftmargin=*,itemsep=1pt]
    \item A non-informative Dirichlet prior pushes mass towards the vertices of the probability simplex $\{0,1\}^{V} \cap \Delta^{V-1}$; since we know that documents are represented by vertices of the probability simplex, this effectively encodes prior knowledge about the generative process directly into the noisy observation model. Encoding prior knowledge about the data this way is a noticeable advantage of PMMs that is not present in other generative modeling frameworks. 
    \item It simplifies the posterior mean dynamics which are given by (more details in section \ref{apx:dir-cat-pmm-update} of the appendix)
\end{enumerate} 
\begin{align}
    \pmb{\mu}_{sc} | (\pmb{\mu}_{s-1,c} = \frac{\mathbbm{1}_{V}}{V}, \pmb{x}) &= \begin{cases}
        \pmb{y}_{sc}^{(1:V)} , \quad \text{if } \pmb{y}_{sc}^{(1:V)} \neq \pmb{0} \\ 
        \frac{\mathbbm{1}_{V}}{V}, \quad \text{if } \pmb{y}_{sc}^{(1:V)} = \pmb{0}
    \end{cases} \\
    &\overset{d}{=} 
    \begin{cases}
         \text{Cat}(\pmb{x}_{c}) \quad \text{w/prob } \alpha_{tc} \\
       \frac{\mathbbm{1}_{V}}{V} \quad \text{w/prob } 1 - \alpha_{tc} 
    \end{cases}  \label{eqn:online_bayes_text1} \\ 
    \pmb{\mu}_{tc} | (\pmb{\mu}_{t-1,c} \neq \frac{\mathbbm{1}_{V}}{V} , \pmb{x}) &= \pmb{\mu}_{t-1,c}. \label{eqn:online_bayes_text2}
\end{align}
We use these updates to derive the PMM objective. 

\parhead{Dirichlet-Categorical Posterior Mean Matching Objective.} As before, we approximate the online Bayesian inference updates from equations (\ref{eqn:online_bayes_text1}) and (\ref{eqn:online_bayes_text2}), using a neural network.
\begin{align}
    \hat{\pmb{x}}_{tc} &= g_{\pmb{\varphi}}(\pmb{\mu}_{t-1}, t) \\ 
    \hat{\pmb{y}}_{sc} &\sim \text{Cat}_{V+1}(\alpha_{sc} \hat{x}_{c}^{(1)}, ..., \alpha_{tc}\hat{x}_{c}^{(V)}, 1- \alpha_{tc}) \label{eqn:apx_online_bayes_text1} \\ 
    \pmb{\mu}_{sc} | (\pmb{\mu}_{s-1,c} = \frac{\mathbbm{1}_{V}}{V}) &= \begin{cases}
         \text{Cat}(\pmb{\hat{x}}_{tc}) \quad \text{w/prob } \alpha_{tc} \\
       \frac{\mathbbm{1}_{V}}{V} \quad \text{w/prob } 1 - \alpha_{tc} 
    \end{cases}.  \label{eqn:apx_online_bayes_text2}
\end{align}
Substituting (\ref{eqn:online_bayes_text1}) and (\ref{eqn:apx_online_bayes_text2}) into the PMM objective, we obtain (see appendix \ref{apx:dir-cat-pmm-update}) 
\begin{align}
    \mathcal{L}_{\text{PMM}} (\pmb{\varphi}) \propto  \nonumber\\ - \sum_{c} t \mathbb{E}_{\substack{ 
       s \sim U(\{1, \dots, t\}) \\ 
       \pmb{x} \sim p^{*}(\pmb{x}) \\ 
       \pmb{\mu}_{s-1} | \pmb{x} 
   }}  \mathbbm{1}\bigg(\pmb{\mu}_{t-1,c} =& \frac{\mathbbm{1}_{V}}{V} \bigg)\alpha_{tc} \log{g_{\pmb{\varphi}}^{(x_{c})}(\pmb{\mu}_{t-1})_{c}} \label{eqn:pmm_non_info}
\end{align}
\parhead{A continuous-time PMM objective for the Dirichlet Categorical Model.} It is relatively straightforward to generalize the PMM objective of the Dirichlet-Categorical model to a continuous-time formulation. This generalization is obtained by taking the continuum limit. We defer the technical details of this formulation to Appendix~\ref{apx:dir-cat-cts}). 

\subsection{Posterior Mean Matching with Other Conjugate Pairs}\label{eqn:alt-pmm-models}
In general, it is possible to apply the same logic we used to derive the Normal-Normal and Dirichlet-Categorical PMMs to other conjugate Bayesian models. Conjugacy is a powerful tool since it allows us to compute the posterior mean of a Bayesian model in closed form. We believe that generalizing Posterior Mean Matching to situations where the posterior mean is not available in closed form represents an exciting avenue for future work. Now that we are equipped with all of the tools necessary to derive PMM models, we briefly state the components of a Gamma-Poisson PMM below and defer the development of the Inverse Gamma-Gamma model to the Appendix. 

\parhead{Data Representation.} The Gamma-Poisson model is suitable to model target distributions $p^{*}(\pmb{x})$ of positive or count data. 

\parhead{Augmented Target Distribution.} Given a sample from the target $\pmb{x} \sim p^{*}(\pmb{x})$, we consider the following noisy observation model $\pmb{y}_{t} | \pmb{x} \sim \text{Pois}(\alpha_{t}\pmb{x})$. These components completely specify the augmented target distribution $p(\pmb{y}_{1:t}) = \int p^{*}(\pmb{x}) \prod_{s} \text{Pois}(\pmb{y}_{t} ; \alpha_{t} \pmb{x}) d\pmb{x}$. 

\parhead{Augmented Bayesian Model.} Given a noisy sample $\pmb{y}_{1:t}$ from the marginal distribution $p(\pmb{y}_{1:t})$, it is possible to recover the sample from the target $\pmb{x}^{*} \sim p^{*}(\pmb{x}) $ by using a  Bayesian model with prior $\pmb{x} \sim \Gamma(\beta_{1}, \beta_{2})$ and likelihood $\pmb{y}_{t} | \pmb{x} \sim \text{Pois}(\alpha_{t}\pmb{x})$. 

\parhead{Online Bayesian Inference Update.} For the Gamma-Poisson PMM model the online Bayesian Inference update is given by 
\begin{align}
     \pmb{\mu}_{s} | \pmb{\mu}_{s-1}, \pmb{x} &\overset{d}{=} \frac{\beta_{2} + \sum_{k=1}^{s-1}\alpha_{k}}{\beta_{2} + \sum_{k=1}^{s}\alpha_{k}}\pmb{\mu}_{s-1} + \frac{\alpha_{s}\pmb{y}_{s}}{\beta_{2} + \sum_{k=1}^{s}\alpha_{k}}, \label{eqn:online-bayes-gamma-pois}
\end{align}
where $\pmb{y}_{s} \sim \text{Pois}(\alpha_{t} \pmb{x})$. 

\section{Diffusion Models and SDEs} \label{sec:sde}

Theorem~\ref{thm:norm-norm} establishes that the posterior mean \( \pmb{\mu}_t \) converges to the true observation  \( \pmb{x} \) as more observations are incorporated. Intuitively, the iterative refinement of online Bayesian inference is analogous to the denoising steps of diffusion models, where each step incrementally reduces noise to approach the underlying data distribution. Here we formalize this intuition by mathematically connecting PMMs and stochastic differential equations (SDEs). 

Specifically, the Bayesian update in the Normal-Normal PMM model can be interpreted as a discrete-time step in a type of diffusion process, with the posterior mean \( \pmb{\mu}_t \) acting as the denoising function steering towards the sample \( \pmb{x} \) from the target distribution \( p^{*}(\pmb{x}) \). Although the continuous-time formulation of the Normal-Normal PMM is, strictly speaking, a diffusion process, we want to emphasize that the behavior and functional form of the SDEs are different from those typically appearing in the literature on diffusion models \citep{song2020score}. We connect the Normal-Normal PMM model to SDEs in the following theorem.
\begin{theorem}\label{thm:online_bayes_diffusion}
    \textbf{(Online Bayesian Inference as a Diffusion Process)} Consider the update rule for the posterior mean \(\pmb{\mu}_t\) given by \eqref{eqn:normal-online-bayes}. Let \(f : [0,1] \to \mathbb{R}^{+}\) be a monotonic function such that $\lim_{t \to 1} \int_{0}^{t} f(\tau) d\tau \to \infty$ and consider a partition of the unit interval \(0 = t_{1} < t_{2} < \ldots < t_{T} = 1\) . Moreover, define the sequence \(\alpha_{1}, \ldots, \alpha_{T}\) in \eqref{eqn:normal-likelihood} by \(\alpha_{s} = f(t_{s}) \delta t_{s}\). In the limit as \(T \to \infty\) and \(\delta t_{s} \to 0\), the discrete updates of \eqref{eqn:normal-online-bayes} converge to a diffusion process defined by the following Stochastic Differential Equation (SDE):
    \begin{align}
        d\pmb{\mu}(t) &= f(t)\frac{(\pmb{x} - \pmb{\mu}(t))}{b + \int_{0}^{t} f(\tau) \, d\tau} dt + \frac{\sqrt{f(t)}}{b + \int_{0}^{t} f(\tau) \, d\tau} d\pmb{W}_{t}, \\
        \pmb{x} &\sim p^{*}(x).
    \end{align}
\end{theorem}

What is surprising is that the continuous-time formulation of the Gamma-Poisson PMM model is also related to SDEs.
\begin{theorem} \label{thm:gamma-pois-sde}
    (Gamma-Poisson SDE) Consider the update rule of the posterior mean \(\mu_t\) for the Gamma-Poisson PMM shown in equation (\ref{eqn:online-bayes-gamma-pois}). Let \(f : [0,1] \to \mathbb{R}^{+}\) and consider \(0 = t_{1} < t_{2} < \ldots < t_{T} = 1\) a partition of the unit interval. Moreover, define the sequence \(\alpha_{1}, \ldots, \alpha_{T}\) of the Gamma-Poison PMM by \(\alpha_{s} = f(t_{s}) \delta t_{s}\). In the continuum limit \(T \to \infty\) and \(\max_{s}\delta t_{s} \to 0\), we have that the discrete updates of $\pmb{\mu}_{t}$ converge to a Merton jump process characterized by the following Stochastic Differential Equation (SDE):
    \begin{align}
        d\pmb\mu(t) &= \left(L'(t) + \frac{A'(t)}{A(t)}\left(\pmb\mu(t) - L(t)\right)\right) dt + A(t) d\pmb{N}(t). 
    \end{align}
    Where $\pmb{N}(t)$ is a Cox Process with random base measure $\pmb{x} dt$ with $\pmb{x} \sim p^{*}(\pmb{x})$, and $A(t) = (\beta_{2} + \int_{0}^{t} f(\tau) d \tau )^{-1}$ and $L(t) = \beta_{1} (\beta_{2} + \int_{0}^{t} f(\tau) d\tau)^{-1}$.
\end{theorem}
The proof is in the Appendix. Theorem \ref{thm:gamma-pois-sde} marks a significant departure from traditional Brownian motion-based generative models that typically appear in the literature on diffusion models \citep{song2020score}. 

\parhead{The Computation / Quality Trade-off.} The connections between PMMs and SDEs established in Theorems \ref{thm:online_bayes_diffusion} and \ref{thm:gamma-pois-sde} allow PMM models to use numerical techniques that have been developed to solve stochastic differential equations over many decades. This connection to SDEs allows us to interpret algorithm \ref{alg:pmm_inference} as using the Euler-Maruyama method to numerically sample paths from an SDE. In the experiments of Section \ref{sec:exper}, we use this connection to SDEs to trade compute for sample quality. 

\section{Experiments}  \label{sec:exper}
We evaluate the performance of Posterior Mean Matching (PMM) models on image and text generation tasks. For image generation, we train Normal-Normal and Gamma-Poisson PMMs on three benchmark datasets: CIFAR-10 \citep{krizhevsky2009learning}, FFHQ-64 \citep{ffhq}, and AFHQv2-64 \citep{choi2020starganv2}. For our text experiments, we evaluate the performance of a Dirichlet-Categorical PMM on the text8 and OpenWebText dataset \citep{OpenWebCorpus, text8dataset}. The following is a summary of our findings:
\begin{itemize}[leftmargin=*,itemsep=1pt]
    \item The Normal-Normal PMM achieves a competitive FID score of 2.18 on CIFAR-10, an FID score that is comparable to most diffusion models.
    \item The Gamma-Poisson PMM achieves an FID score of 4.36 on CIFAR-10. This score is lower than other diffusion models based on the Poisson likelihood \citep{learningtojump, santos2023blackout}. 
    \item Using the SDE interpretation of PMMs, we show that the FID scores of the Normal PMM degrade marginally when using a reduced number of function evaluations. Notably, on CIFAR-10 decreasing the number of function evaluations from $5000$ to $166$ (a factor of 30) reduces the FID score from $2.18$ to $2.79$. 
    \item On OpenWebText the Dirichlet-Categorical PMM achieves a generative perplexity of 37.06 and 42.58 using top-350 and top-500 sampling, respectively,  demonstrating performance on par with current non-autoregressive diffusion-based language models  \citep{sedd, mdlm, shi2024simplified}. 
    \item On text8, PMM achieves a bits per character (BPC) of 1.29, better than non-autoregressive language models based on diffusion. It narrows the gap to autoregressive language models, which achieve a BPC of 1.23.
\end{itemize}

\subsection{Image Generation Tasks} \label{sec:img-exps}
\parhead{Neural Network Architecture.} In all of our experiments, we use an open-source implementation of the DDPM++ architecture \citep{Karras2022edm, dhariwal2021diffusionmodelsbeatgans}.
\begin{table}[t]
    \centering
    \caption{If FID scores are available for different sources, we report both scores (lower is better). The FID scores for these models may be found in \citet{Karras2022edm, song2020score}. All of our experiments make use of class conditioning. The Normal-Normal PMM also uses adaptive data augmentation.}
    \label{tab:fid}
    \begin{threeparttable}
        \scriptsize % Further reduce font size for compactness
        \setlength{\tabcolsep}{4pt} % Reduce horizontal padding
        \renewcommand{\arraystretch}{1.2} % Adjust vertical padding
        \begin{tabularx}{\columnwidth}{@{}l S[table-format=2.2] S[table-format=3.2] S[table-format=3.2]@{}} 
            \toprule
            \textbf{Method} & \textbf{Cifar-10} & \textbf{FFHQ-64} & \textbf{AFHQv2-64} \\
            \midrule
            \textbf{DDPM} & 3.17 & \NA & \NA\\ 
            \textbf{DDPM++} & 2.78 & \NA & \NA \\
            \textbf{DDPM ++ (VP)} & 2.18$^{*}$/2.55 & 3.13$^{*}$ & 2.43$^{*}$ \\
            \textbf{DDPM ++ (VE)} & 2.48$^{*}$ & 22.53$^{*}$ & 23.12$^{*}$ \\ 
            \textbf{NSCN++ (VE)} & 2.38 & \NA  & \NA\\ 
            \textbf{NCSN ++ (VE, deep)} &  2.20 & \NA & \NA \\ 
            \textbf{DDPM++ (VP, deep)} & 2.41 & \NA & \NA \\ 
            \textbf{PFGM} & 2.48$^{*}$ &  \NA & \NA \\ 
            \textbf{PFGM, deep} & 2.35$^{*}$ & \NA & \NA \\ 
            \textbf{Style GAN w/ADA} & 2.42 & \NA & \NA\\ 
            \textbf{SOTA Diffusion} & 1.79 & 2.39  & 1.96  \\
            \textbf{Normal PMM (ours)} & \textbf{2.18} & 3.41 &  2.48 \\
            % Gamma-Poisson PMM (our) & \textbf{4.36} & 7.95 & 5.93 &\\
            \bottomrule            
        \end{tabularx}
        \begin{tablenotes}
            \small
            \item \textbf{NFE}: Neural Function Evaluations. 
            \item \textbf{(deep)}: Methods with this marker use deeper networks. 
            \item *: Indicates use of a higher order solver like RK-45 or Heun.        
        \end{tablenotes} 
    \end{threeparttable}
\end{table}

\begin{table}[t]
    \centering
    \caption{FID scores for the Normal-Normal PMM as a function of neural function evaluations (NFEs).}
    \label{tab:nfe-fid}
    \begin{threeparttable}
        \scriptsize % Further reduce font size for compactness
        \setlength{\tabcolsep}{3pt} % Adjust column spacing for compactness
        \renewcommand{\arraystretch}{1.2} % Adjust vertical padding
        \begin{tabularx}{\columnwidth}{@{}c *{6}{S[table-format=3.2]}@{}} 
            \toprule
            \textbf{Method / NFE} & \textbf{100} & \textbf{166} & \textbf{500} & \textbf{1k} & \textbf{3k} & \textbf{5k} \\
            \midrule
            \textbf{Cifar10} & 3.98 & 2.79 & 2.46 & 2.33 & 2.28 & 2.18 \\
            \textbf{AFHQ-v2} & \NA  & 3.04 & 2.62 & 2.52 & 2.48 & \NA \\ 
            \textbf{FFHQ} & \NA &  5.76 & 3.89 & 3.65 & 3.41 & \NA \\ 
            \bottomrule
        \end{tabularx}
    \end{threeparttable}
\end{table}

\parhead{Discussion.} The results of our experiments are shown in Table \ref{tab:fid}. Unless otherwise stated, the methods in Table \ref{tab:fid} use very similar Neural Network architectures to the ones used in our experiments. A notable exception to this rule is Style-GAN \citep{ffhq}. We report the performance of Style-GAN to paint a more complete picture of the performance of state-of-the-art models that are not based on diffusion. 

We measure the quality of the generated images by computing the FID score on a sample of $50,000$ images. We also evaluate the Normal-Normal PMM on AFHQ-v2 and FFHQ-64, two higher-resolution datasets consisting of images of animals and humans, respectively. The performance of the Normal-Normal PMM on these datasets is also comparable to other popular diffusion models (see Table \ref{tab:fid}). 

We compare the performance of the Gamma-Poisson PMM against other Poisson diffusion models. The Gamma-Poisson PMM achieves an FID score of $4.36$, which is a better score than previous generative models that use the Poisson distribution. The details of the Gamma-Poisson model are shown in Appendix.  
In all of our image generation experiments, we choose an exponential schedule for the $\alpha_{t}$'s in the noise model (see Appendix \ref{app:hyperparameters}) 

\begin{figure}[t]
    \centering
    % Row for CIFAR10
    \begin{subfigure}[b]{0.48\textwidth}  % Model A - CIFAR10
        \centering
        \includegraphics[width=\textwidth]{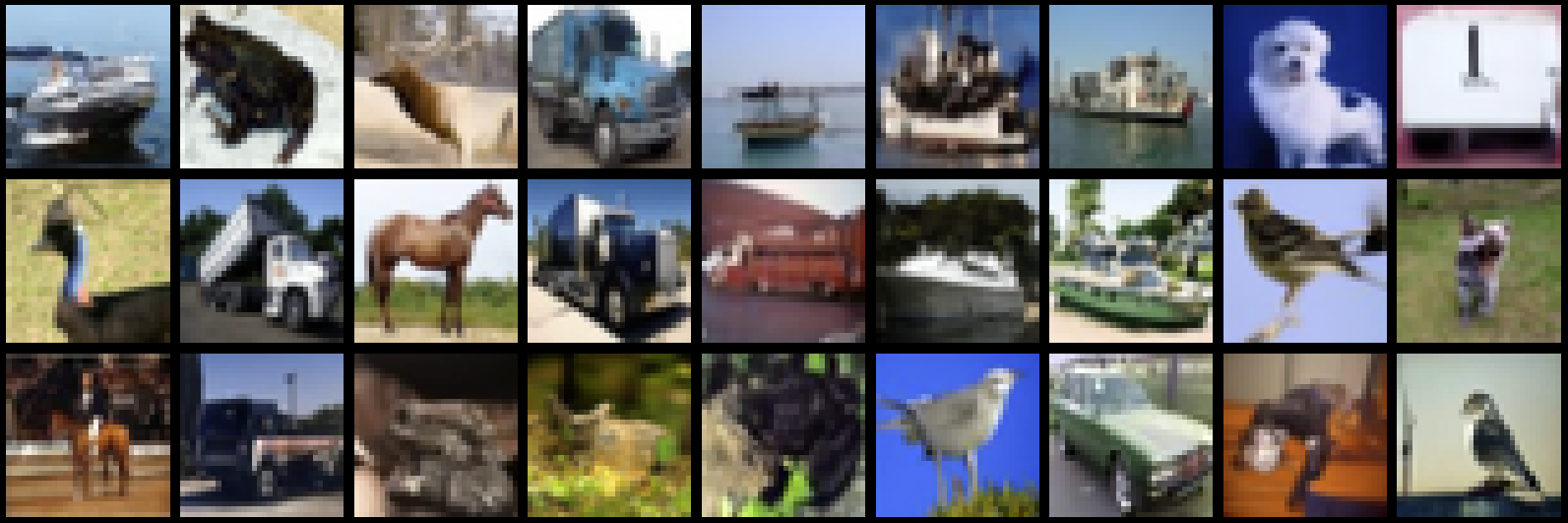}
        \caption{CIFAR 10 samples. FID = 2.46 with 500 NFEs}
        \label{fig:pmm_cifar10}
    \end{subfigure}
    \hfill
    \begin{subfigure}[b]{0.48\textwidth}  % Model B - CIFAR10
        \centering
        \includegraphics[width=\textwidth]{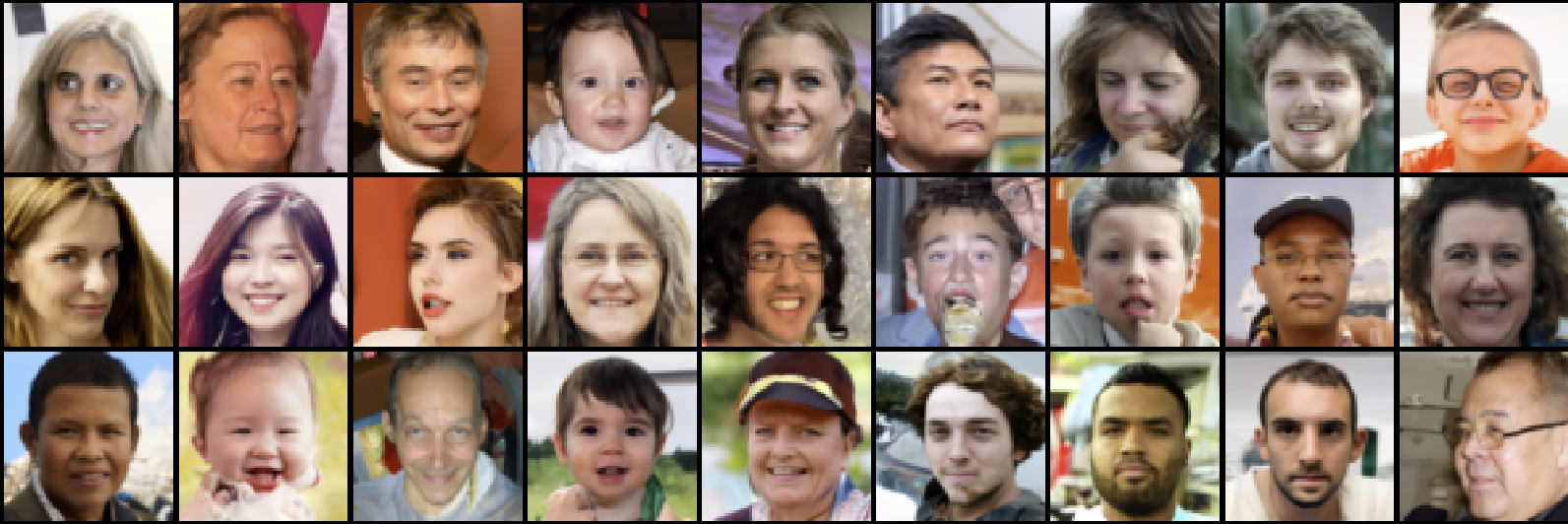}
        \caption{FFHQ samples. FID = 3.89 with 500 NFEs}
        \label{fig:pmm_ffhq}
    \end{subfigure}

    \caption{Comparison of sample generations for the Normal-Normal PMM model across CIFAR10 and FFHQ datasets. See Appendix \ref{app:figures} for a larger sample of generated images.}
    \label{fig:img-generation}
\end{figure}

\subsection{Language Modeling Tasks}
For the text experiments, we fit a continuous-time Dirichlet-Categorical PMM model using a non-informative prior on two datasets: text8 and OpenWebText. The text8 dataset consists of the first 100M characters of cleaned English Wikipedia text, while OpenWebText is a large-scale corpus derived from web pages shared on Reddit with high engagement. On text8 we find that the Dirichlet-Categorical PMM outperforms all of the existing non-autoregressive baselines (see Table \ref{tab:text8_bpc}).  Dirichlet-Categorical PMMs also achieve competitive Generative Perplexity on OpenWebText compared to other non-autoregressive (see Table \ref{tab:owt}). 

\parhead{Language Modeling} We evaluate our Dirichlet-Categorical PMM against several baselines, including traditional autoregressive models and recent non-autoregressive approaches. On the text8 dataset, our model achieves a bits per character (BPC) of 1.29. 
\begin{table}[t]
    \centering
    \caption{Bits per Character (BPC) Performance}
    \label{tab:text8_bpc}
    \begin{threeparttable}
        \scriptsize % Reduce font size for compactness
        \setlength{\tabcolsep}{4pt} % Reduce horizontal padding
        \renewcommand{\arraystretch}{1.2} % Adjust vertical padding
        \begin{tabularx}{\columnwidth}{@{}l S[table-format=1.2]@{}}
            \toprule
            \textbf{Model} & \textbf{BPC} \\
            \midrule
            Autoregressive (GPT-2) & 1.23 \\
            D3PM (Uniform) & 1.61 \\
            D3PM (Absorb) & 1.45 \\
            SEDD (Absorb) & 1.39 \\ 
            Bayesian Flow Nets & 1.41 \\ 
            GenMD4 & 1.34 \\
            \midrule
            \textbf{Dirichlet-Categorical PMM (Ours)} & \textbf{1.29} \\
            \bottomrule
        \end{tabularx}
    \end{threeparttable}
\end{table}

\begin{table}[t]
    \centering
    % \caption{Generation Perplexity (CL=1024)}
    \caption{Generative Perplexity (\textbf{PPL}) measured at context length (CL) 1024 relative to GPT-2 large.}
    \label{tab:owt}
    \begin{threeparttable}
        \scriptsize % Reduce font size for compactness
        \setlength{\tabcolsep}{4pt} % Reduce horizontal padding
        \renewcommand{\arraystretch}{1.2} % Adjust vertical padding
        \begin{tabularx}{\columnwidth}{@{}l S[table-format=2.2]@{}}
            \toprule
            \textbf{Model} & \textbf{PPL} \\
            \midrule
            AR & 20.98 \\
            SEDD & 30.96 \\
            MDLM & 31.69 \\
            \midrule
            \textbf{Dirichlet-Categorical PMM} & 42.58 \\
            \bottomrule
        \end{tabularx}
    \end{threeparttable}
\end{table}

\parhead{Unconditional Language Generation} 
Following previous work \citep{sedd, mdlm}, we assess the quality of unconditional text output by computing generative perplexity relative to the \texttt{gpt2-large} language model \citep{radford2019language}. While generative perplexity is a standard metric for evaluating traditional autoregressive language models, its estimation for non-autoregressive models is more complex due to their inherently different generation processes. To ensure a fair comparison, we sample tokens using top-500 sampling and generate 1024 sequences of length 1024 from each model, using at most 1000 network evaluations for the non-autoregressive models. The resulting generative perplexities are shown in Table \ref{tab:owt}. where our model attains a competitive generative perplexity \footnote{PMM results were obtained using a different tokenizer, a smaller model with gpt2-based architecture, and without applying exponential moving averages (EMA).} with other non-autoregressive models. 

While this is an exciting finding, it is important to note that evaluating unconditional non-autoregressive language models remains challenging, and it is known that benchmarks like generative perplexity are susceptible to manipulation through temperature annealing techniques \citep{sedd}. For this reason, we encourage readers to qualitatively assess the text generation capabilities of each model themselves by looking at the samples provided in the Appendix. These samples showcase our model's ability to generate coherent and diverse text across various topics and styles.

\section{Conclusion and Future Work}
In this paper, we introduced \textbf{Posterior Mean Matching (PMM)}, a novel and flexible framework for generative modeling grounded in Bayesian inference. PMM leverages conjugate pairs of distributions to model complex data distributions across various modalities, offering an alternative to traditional diffusion models. Through comprehensive experiments, we demonstrated the efficacy of PMM in both image and language generation tasks.

\section{Ackowledgements}
This work was conducted under the U.S. Department of Energy, National Nuclear Security Administration's Office of Defense Nuclear Nonproliferation Research and Development (NA-22) Steel Thread Venture. David Blei acknowledges support from the Simons Foundation, the National Science Foundation (grant numbers IIS-2127869 and DMS-2311108), and the Office of Naval Research (grant number N000142412243). Yixin Wang is supported by the National Science Foundation (grant numbers 2231174, 2310831, and 2428059), the Office of Naval Research (grant number N000142312590), and the Michigan Institute for Data Science Propelling Original Data Science (PODS) grant. 

This research used resources provided by the Los Alamos National Laboratory Institutional Computing Program, which is supported by the U.S. Department of Energy National Nuclear Security Administration under Contract No. 89233218CNA000001.

We are deeply grateful to members of the Blei Lab for their invaluable contributions: Alessandro Grande for his thoughtful feedback and discussions on the early drafts, and Nicholas Beltran for reviewing an early version of the paper. We also thank Luhuan Wu, Sweta Karlekar, and Eli Weinstein for their insightful discussions during the initial stages of the project. We extend our appreciation to our colleagues at Los Alamos National Laboratory (LANL), Dr. Yen Ting Lin and Nihar Mauskar, for their thoughtful input, and to Carolyn Connor for her assistance in accessing the Venado computing cluster.

\bibliography{bibliography}

% \begin{document}

\onecolumn
\aistatstitle{Posterior Mean Matching: Generative Modeling with Online Bayesian Inference \\
Supplementary Materials}
\appendix

\setcounter{theorem}{0}
\renewcommand{\thetheorem}{\Alph{section}.\arabic{theorem}}

\setcounter{figure}{0}  % Reset figure counter
\renewcommand{\thefigure}{\Alph{section}.\arabic{figure}}  % Number figures as A.1, A.2, etc.

\section{Appendix: Details for PMM Models}
% \tableofcontents

In this section we ouline the details for the following PMM models 
\begin{itemize}
    \item Normal-Normal 
    \item Gamma-Poisson
    \item Dirichlet-Categorical
    \item InverseGamma-Gamma 
\end{itemize}
For each of these models we work out the following 
\begin{enumerate}
    \item The online Bayesian update.
    \item A proof of consistency. 
    \item The closed-form of the PMM objective. 
    \item If applicable, a continuous time formulations and/or connections to stochastic differential equations. 
\end{enumerate}
All of the Bayesian models we consider form conjugate pairs in the exponential family. 

We summarize these details in Table in Section \ref{sec:reference}. 

\vfill

\newpage 
% \begin{landscape}
\subsection{Reference Table}\label{sec:reference}
\begin{figure}[!h]
    \centering
    \includegraphics[angle=90, width=0.65\linewidth]{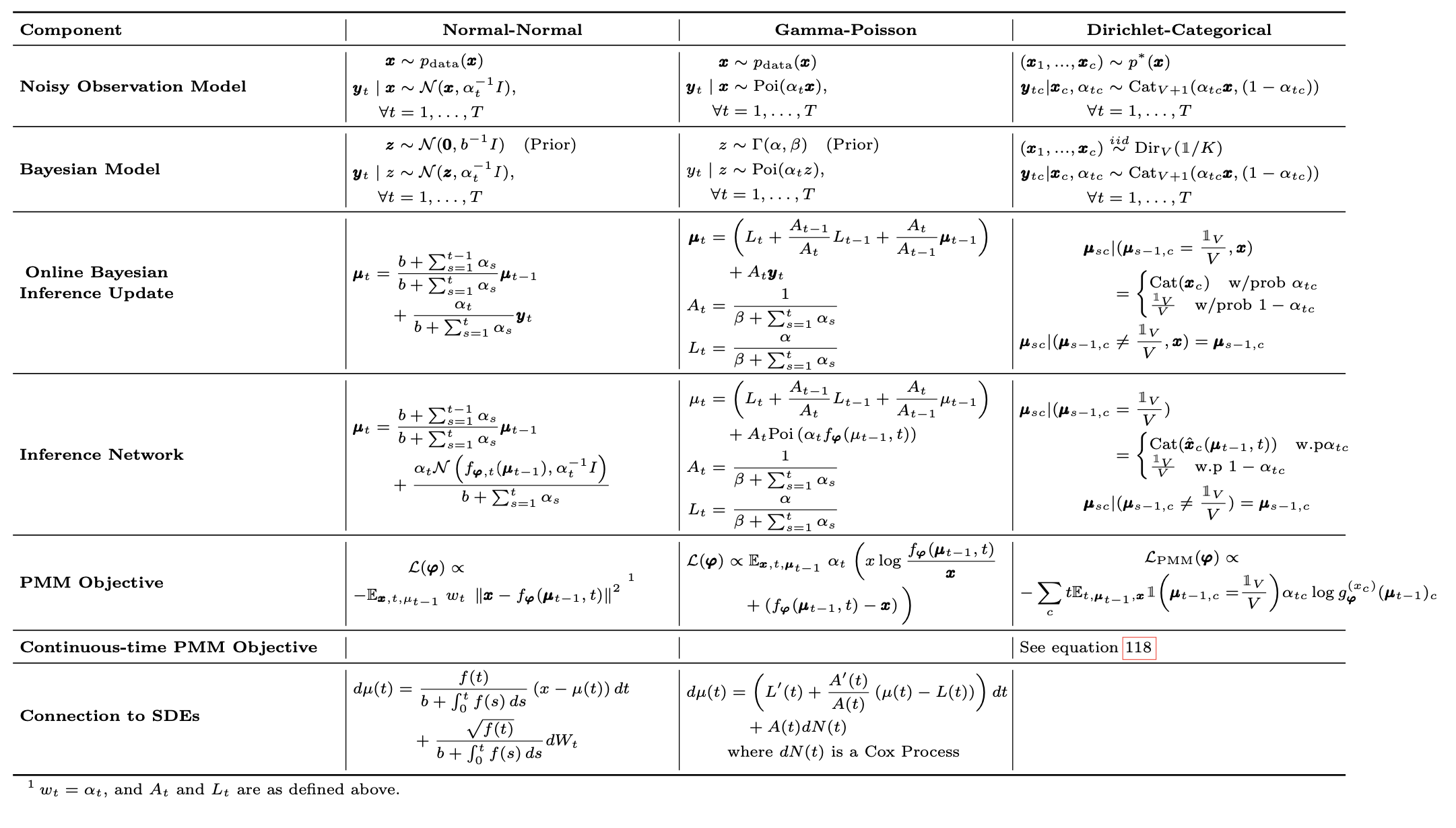}
    % \caption{Enter Caption}
    \label{fig:reference-table}
\end{figure}

\newpage 
\subsection{Normal-Normal Model}
\subsubsection{Online Bayesian Update} \label{apx:normal-online}
For the Normal-Normal Bayesian model:
\begin{align}
     \pmb{x} &\sim \mathcal{N}(\pmb{0}, \beta^{-1}I) \\
     \pmb{y}_{s} | \pmb{x} &\sim \mathcal{N}( \pmb{x}, \alpha_s^{-1}I) , \quad \forall s \in \{1, ..., t\} \label{eqn:bayes2app}
\end{align}
the posterior distribution under this model is given by 
\begin{align}
    p( \pmb{x} | \pmb{y}_1, \dots, \pmb{y}_t) = 
    \mathcal{N}\left( \pmb{x} \ \Bigg| \ \frac{ \sum_{i=1}^{t} \alpha_{i}\pmb{y}_{i}}{\beta + \sum_{i=1}^{t}\alpha_{i}}, \ \left(\beta + \sum_{i=1}^{t}\alpha_{i}\right)^{-1} I \right)
\end{align}
From this we can read off the posterior mean 
\begin{align}
        \pmb{\mu}_{t} &\overset{d}{=} \frac{\sum_{i=1}^{t} \alpha_{i}\pmb{y}_{i}}{\beta + \sum_{i=1}^{t}\alpha_{i}}.
\end{align}
The above expression can further be rewritten as
\begin{align}
    \pmb{\mu}_{t} &= \frac{ \sum_{i=1}^{t-1} \alpha_{i}\pmb{y}_i}{\beta + \sum_{i=1}^{t}\alpha_{t}} + \frac{ \alpha_{t}}{\beta + \sum_{i=1}^{t}\alpha_{t}}\pmb{y}_t \label{eqn:online-normal} \\
    &= \frac{\beta + \sum_{i=1}^{t-1}\alpha_{t}}{\beta + \sum_{i=1}^{t}\alpha_{t}} \pmb{\mu}_{t-1} + \frac{ \alpha_{t}}{\beta + \sum_{i=1}^{t}\alpha_{t}}\pmb{y}_t
\end{align}
giving us the online Bayesian update from Equation \ref{eqn:normal-online-bayes}.
\newpage 
\subsubsection{Consistency of the Normal-Normal Model}
\begin{theorem} \label{thm:norm-norm-app}
(Concentration of posterior mean) Let \( \{ \pmb{y}_1, \ldots, \pmb{y}_t \} \) be observations generated according to equation (\ref{eqn:norm-norm-pmm-marginal}).  Suppose \( \alpha_t \) a known, positive, increasing sequence satisfying \( \lim_{t \to \infty}  \sum_{s=1}^{t} \alpha_s = \Omega(t^{1+\eta}) \) for all $\eta > 0$. Then, the posterior mean \( \pmb{\mu}_t \) of  the Bayesian model in equations (\ref{eqn:normal-prior}) and (\ref{eqn:normal-likelihood}) is consistent, namely:
\begin{align}
    \lim_{t \to \infty} \pmb{\mu}_t = \pmb{x}, \quad \text{almost surely},
\end{align}
with respect to the joint distribution of \( (\pmb{x}, \pmb{y}_1, \pmb{y}_2, \ldots) \) in equation (\ref{eqn:norm-norm-pmm}). 
\end{theorem}
\begin{proof}
    We give a proof for the one dimensional case with the understanding that it extends to the multidimensional case by applying the same argument to each coordinate individually. The posterior mean under this Bayesian Model is given by 
    \begin{align}
    \mu_{t} &\equiv \mathbb{E}(x | y_{1:t}) \\
    &= \frac{\sum_{s=1}^{t}\alpha_{s} y_s}{b + \sum_{s=1}^{t} \alpha_s} \\
    &= \mathcal{N}\left(\frac{\sum_{s=1}^{t}\alpha_{s}}{\beta + \sum_{s=1}^{t} \alpha_s} x, \frac{\sum_{s=1}^{t} \alpha_{s}}{(\beta + \sum_{s=1}^{t} \alpha_s)^{2}}\right)
\end{align}
    Assuming that $\sum_{s=1}^{t} \alpha_{s} \to \infty$, it we see that $\frac{\sum_{s=1}^{t}\alpha_{s}}{\beta + \sum_{s=1}^{t} \alpha_{s}} \to 1 $ while $\frac{\sum_{s=1}^{t}\alpha_{s}}{(\beta + \sum_{s=1}^{t} \alpha_{s})^{2}} \to 0$. Putting these two things together we see that the mean of $\mathbb{E}(\mu_{t}) \to x$ and $\text{Var}(\mu_{t})  \to 0$. This implies that $\mu_{t} \overset{\mathbb{P}}{\to} x$. 
    
    To obtain almost sure convergence, consider the event $A_{t} = \{ \abs{\mu_t - a_{t}x } > \epsilon \}$ where $a_{t} = \frac{\sum_{s=1}^{t}\alpha_{s}}{\beta + \sum_{s=1}^{t}\alpha_{t}} \to 1$. Using Chebyshev's inequality, we obtain 
\begin{align}
        \mathbb{P}(A_{t}) &\leq \epsilon^{-2} \text{Var}(\mu_{t}) \\
        &= \epsilon^{-2} \frac{\sum_{s=1}^{t} \alpha_{s}}{(\beta + \sum_{s=1}^{t} \alpha_s)^{2}} \\
        &= \epsilon^{-2}O\left(\left(\sum_{s=1}^{t} \alpha_{s}\right)^{-1}\right) \\
        &= \epsilon^{-2} O\left(\frac{1}{t^{1+\eta}}\right)
\end{align}
Thus, there's a constant $C$ such that 
\begin{align}
    \sum_{t}\mathbb{P}(A_{t}) &\lesssim \epsilon^{-2} \sum_{s}\frac{1}{s^{1+\eta}} \\
    &= C \epsilon^{-2}\zeta(1 + \eta) \\ 
    &< \infty
\end{align}
Where $\zeta$ is the Riemann-Zeta function. By the Borel-Cantelli lemma, it follows that with probability $1$, the events $A_{t}$ happen for at most finitely many $t$. In other words, for all $\epsilon > 0$, there is a $T$ such that for all $t > T$, $\abs{\mu_{t} - a_{t}x} < \epsilon$. Therefore $\lim_{t \to \infty} \abs{\mu_{t} - a_{t} x} = \lim_{t\to\infty} \abs{\mu_{t} - x} < \epsilon$. Since $\epsilon$ is an arbitrary real number, it follows that $\lim_{t\to\infty} \abs{\mu_{t} - x} = 0$ with probability $1$, namely $\mu_{t} \overset{a.s}{\to} x$, as desired. 
\end{proof}

\newpage 
\subsubsection{PMM Objective} \label{apx:normal-pmm-obj}
The PMM objective for the Normal-Normal model is proportional to 
\begin{align}
    \mathcal{L}_{\text{PMM}}(\pmb{\varphi}) &= \int p(\pmb{\mu}_{1}, ..., \pmb{\mu}_{t}) \log{\frac{p(\pmb{\mu}_{1}, ..., \pmb{\mu}_{t})}{q_{\pmb{\varphi}}(\pmb{\mu}_{1}, ..., \pmb{\mu}_{t})}}d \pmb{\mu}_{1:t} \\
    &\propto - \int p(\pmb{\mu}_{1}, ..., \pmb{\mu}_{t})  \log{q_{\pmb{\varphi}}(\pmb{\mu}_{1}, ..., \pmb{\mu}_{t})} d\pmb{\mu}_{1:t} \\ 
    &= - \int \int p(\pmb{x}, \pmb{\mu}_{1}, ..., \pmb{\mu}_{t})  \log{q_{\pmb{\varphi}}(\pmb{\mu}_{1}, ..., \pmb{\mu}_{t})} d\pmb{\mu}_{1:t} \\ 
    &= -\mathbb{E}_{\pmb{x} \sim p^{*}(\pmb{x})} \mathbb{E}_{\pmb{\mu}_{1:t} | \pmb{x}} \log{\prod_{s=1}^{t} q_{\pmb{\varphi}}(\pmb{\mu}_{s} | \pmb{\mu}_{s-1})} \\ 
    &= -\mathbb{E}_{\pmb{x} \sim p^{*}(\pmb{x})} \sum_{s} \mathbb{E}_{\pmb{\mu}_{s}, \pmb{\mu}_{s-1} | \pmb{x}} \log{ q_{\pmb{\varphi}}(\pmb{\mu}_{s} | \pmb{\mu}_{s-1})}
\end{align}
From Section \ref{apx:normal-online} we know that 
\begin{align}
    \pmb{\mu}_{s-1} | \pmb{x} &\overset{d}{=} \frac{\sum_{k=1}^{s-1} \alpha_{k}\pmb{y}_{k}}{\beta + \sum_{k=1}^{s-1}\alpha_{k}}
\end{align}
Since $\pmb{y}_{k} \sim \mathcal{N}(\pmb{x}, \alpha_{k}^{-1}I)$, it follows that 
\begin{align}
    \pmb{\mu}_{s-1} | \pmb{x} &\overset{d}{=} \frac{\sum_{k=1}^{s-1} \alpha_{k}\mathcal{N}(\pmb{x}, \alpha_{k}^{-1} I)}{\beta + \sum_{k=1}^{s-1}\alpha_{k}} \\ 
    &= \frac{\mathcal{N}\left(\sum_{k=1}^{s-1} \alpha_{k} \pmb{x} , \sum_{k=1}^{s-1} \alpha_{k} I \right)}{\beta + \sum_{k=1}^{s-1}\alpha_{k}} \\ 
    &= \mathcal{N}\left(\frac{\sum_{k=1}^{s-1} \alpha_{k}}{\beta + \sum_{k=1}^{s-1} \alpha_{k}} \pmb{x} ,  \frac{\sum_{k=1}^{s-1} \alpha_{k}}{(\beta + \sum_{k=1}^{s-1} \alpha_{k})^{2}} I\right)
\end{align}
Moreover, using the online update of equation \ref{eqn:online-normal} the conditional distribution of $\pmb{\mu}_{t}$ given $\pmb{\mu}_{t-1}$ and $\pmb{x}$ is given by 
\begin{align}
    \pmb{\mu}_{s} | \pmb{\mu}_{s-1} ,\pmb{y}_{s} &= \frac{\beta + \sum_{k=1}^{s-1} \alpha_{k}}{\beta + \sum_{k=1}^{s}\alpha_{k}}\pmb{\mu}_{s-1} + \frac{\alpha_{s} \pmb{y}_{s}}{\beta + \sum_{k=1}^{s}\alpha_{k}} \\
    &\overset{d}{=}\frac{\beta + \sum_{k=1}^{s-1} \alpha_{k}}{\beta + \sum_{k=1}^{s}\alpha_{k}}\pmb{\mu}_{s-1} + \frac{\alpha_{s}\mathcal{N}(\pmb{x}, \alpha_{s}^{-1} I )}{\beta + \sum_{k=1}^{s}\alpha_{k}}  \\ 
    &= \frac{\beta + \sum_{k=1}^{s-1} \alpha_{k}}{\beta + \sum_{k=1}^{s}\alpha_{k}}\pmb{\mu}_{s-1} + \frac{\mathcal{N}(\alpha_{s}\pmb{x}, \alpha_{s}I)}{\beta + \sum_{k=1}^{s}\alpha_{k}} \\
    &= \mathcal{N}\left(\frac{\beta + \sum_{k=1}^{s-1} \alpha_{k}}{\beta + \sum_{k=1}^{s}\alpha_{k}}\pmb{\mu}_{s-1} + \frac{\alpha_{s}}{\beta + \sum_{k=1}^{s}\alpha_{k}}\pmb{x} , \frac{\alpha_{s}}{(\beta + \sum_{k=1}^{s}\alpha_{k})^{2}}\right)
\end{align}
Using this, the distribution of $q_{\pmb{\varphi}}(\pmb{\mu}_{s} | \pmb{\mu}_{s-1})$ defined by equations (\ref{eqn:noisy_model_online_approx}) and (\ref{eqn:online_bayes_approx}) is given by 
\begin{align}
    \pmb{\mu}_{s} | \pmb{\mu}_{s-1} &\overset{d}{=} \frac{\beta + \sum_{k=1}^{s-1} \alpha_{k}}{\beta + \sum_{k=1}^{s}\alpha_{k}}\pmb{\mu}_{s-1} + \frac{\mathcal{N}(\alpha_{s}g_{\pmb{\varphi}}(\pmb{\mu}_{s-1}, t), \alpha_{s})}{\beta + \sum_{k=1}^{s}\alpha_{k}} \\ 
    &\overset{d}{=} \mathcal{N}\left(\frac{\beta + \sum_{k=1}^{s-1} \alpha_{k}}{\beta + \sum_{k=1}^{s}\alpha_{k}}\pmb{\mu}_{s-1} +\frac{\alpha_{s}}{\beta + \sum_{k=1}^{s}\alpha_{k}}g_{\pmb{\varphi}}(\pmb{\mu}_{s-1}, s) , \frac{\alpha_{s}}{(\beta + \sum_{k=1}^{s}\alpha_{k})^{2}}\right)
\end{align}
Substituting these into the PMM objective we obtain 
\begin{align}
    \mathcal{L}_{\text{PMM}}(\pmb{\varphi}) &=  -\mathbb{E}_{\pmb{x} \sim p^{*}(\pmb{x})} \sum_{s} \mathbb{E}_{\pmb{\mu}_{s}, \pmb{\mu}_{s-1} | \pmb{x}} \log{ q_{\pmb{\varphi}}(\pmb{\mu}_{s} | \pmb{\mu}_{s-1})} \\ 
    &= \mathbb{E}_{\pmb{x} \sim p^{*}(\pmb{x})} \sum_{s} \mathbb{E}_{\pmb{\mu}_{s}, \pmb{\mu}_{s-1} | \pmb{x}} \frac{(\beta + \sum_{k=1}^{s}\alpha_{k})^{2}}{\alpha_{s}} \norm{\pmb{\mu}_{s} - \frac{\beta + \sum_{k=1}^{s-1} \alpha_{k}}{\beta + \sum_{k=1}^{s}\alpha_{k}}\pmb{\mu}_{s-1} -\frac{\alpha_{s}g_{\pmb{\varphi}}(\pmb{\mu}_{s-1}, s)}{\beta + \sum_{k=1}^{s}\alpha_{k}}}_{2}^{2} 
\end{align}
Using the reparametrization trick, we substitute 
\begin{align}
    \pmb{\mu}_{s} | \pmb{\mu}_{s-1}, \pmb{x} &= \frac{\beta + \sum_{k=1}^{s-1} \alpha_{k}}{\beta + \sum_{k=1}^{s}\alpha_{k}}\pmb{\mu}_{s-1} +\frac{\alpha_{s}\pmb{x}}{\beta + \sum_{k=1}^{s}\alpha_{k}} + \frac{\sqrt{\alpha_{s}}}{\beta + \sum_{k=1}^{s}\alpha_{k}}\pmb{\epsilon}_{s}
\end{align}
into the loss to obtain 
\begin{align}
    \mathcal{L}_{\text{PMM}}(\pmb{\varphi}) &= \mathbb{E}_{\pmb{x} \sim p^{*}(\pmb{x})} \sum_{s} \mathbb{E}_{\pmb{\mu}_{s-1} | \pmb{x}} \mathbb{E}_{\pmb{\epsilon}_{s} \sim \mathcal{N}(0,I)} \frac{(\beta + \sum_{k=1}^{s}\alpha_{k})^{2}}{\alpha_{s}} \norm{\frac{\alpha_{s}}{\beta + \sum_{k=1}^{s}\alpha_{k}} (\pmb{x} - g_{\pmb{\varphi}}(\pmb{\mu}_{s-1}, s)) + \frac{\sqrt{\alpha_{s}}}{\beta + \sum_{k=1}^{s}\alpha_{k}}\pmb{\epsilon}_{s} }_{2}^{2} \\
    &\propto \mathbb{E}_{\pmb{x} \sim p^{*}(\pmb{x})} \sum_{s} \mathbb{E}_{\pmb{\mu}_{s-1} | \pmb{x}} \mathbb{E}_{\pmb{\epsilon}_{s} \sim \mathcal{N}(0,I)} \frac{(\beta + \sum_{k=1}^{s}\alpha_{k})^{2}}{\alpha_{s}} \norm{\frac{\alpha_{s}}{\beta + \sum_{k=1}^{s}\alpha_{k}} (\pmb{x} - g_{\pmb{\varphi}}(\pmb{\mu}_{s-1}, s))}_{2}^{2} \\ 
    &= \mathbb{E}_{\pmb{x} \sim p^{*}(\pmb{x})} \sum_{s} \mathbb{E}_{\pmb{\mu}_{s-1} | \pmb{x}}  \alpha_{s} \norm{\pmb{x} - g_{\pmb{\varphi}}(\pmb{\mu}_{s-1}, s)}_{2}^{2} \\ 
    &= t \cdot  \mathbb{E}_{s \sim U(\{1, ..., t\}) }\mathbb{E}_{\pmb{x} \sim p^{*}(\pmb{x})}  \mathbb{E}_{\pmb{\mu}_{s-1} | \pmb{x}}  \alpha_{s} \norm{\pmb{x} - g_{\pmb{\varphi}}(\pmb{\mu}_{s-1}, s)}_{2}^{2} 
\end{align}
For the schedules that we use in this paper, this objective assigns less weight to timesteps containing less information about $\pmb{x}$ (i.e. those with large $\alpha_{s}$) 

We parametrize the neural network using $g_{\pmb{\varphi}}(\pmb{\mu}_{s-1}, s) = \frac{(\beta + \sum_{k=1}^{s}\alpha_{k})}{\sum_{k=1}^{s}\alpha_{k}} \pmb{\mu}_{t-1} - \frac{1}{\sqrt{\sum_{k=1}^{s}\alpha_{k}}}\epsilon_{\pmb{\varphi}}(\pmb{\mu}_{t-1}, t)$. This parametrization is given motivated by using the reparametrization trick to establish a relationship between $\pmb{x}$ and $\pmb{\mu}_{s-1}$, which is given by $\pmb{x} = \frac{\beta + \sum_{k=1}^{s-1}\alpha_{k}}{\sum_{k=1}^{s-1}\alpha_{k}} \pmb{\mu}_{s-1} - \frac{1}{\sqrt{\sum_{k=1}^{s-1}\alpha_{k}}} \pmb{\epsilon}$ where $\pmb{\epsilon} \sim \mathcal{N}(0, I)$. Using this parametrization the PMM objective is given by 
\begin{align}
     \mathcal{L}_{\text{PMM}}(\pmb{\varphi})  &\propto t \cdot  \mathbb{E}_{s \sim U(\{1, ..., t\}) }\mathbb{E}_{\pmb{x} \sim p^{*}(\pmb{x})}  \mathbb{E}_{\pmb{\mu}_{s-1} | \pmb{x}}  \frac{\alpha_{s}}{\sum_{k=1}^{s}\alpha_{k}} \norm{\pmb{\epsilon} - \pmb{\epsilon}_{\pmb{\varphi}}(\pmb{\mu}_{s-1},s) }_{2}^{2} 
\end{align}
For simplicity, the experiments in this paper drop the $\frac{\alpha_{s}}{\sum_{k=1}^{s}\alpha_{k}}$ weighting factor and rewrite the PMM loss. 
\begin{align}
     \mathcal{L}_{\text{PMM reweighted}}(\pmb{\varphi}) &= t \cdot  \mathbb{E}_{s \sim U(\{1, ..., t\}) }\mathbb{E}_{\pmb{x} \sim p^{*}(\pmb{x})}  \mathbb{E}_{\pmb{\mu}_{s-1} | \pmb{x}} \sum_{k=1}^{s} \alpha_{k} \norm{\pmb{x} - g_{\pmb{\varphi}}(\pmb{\mu}_{s-1}, s)}_{2}^{2} \\
     &= t \cdot  \mathbb{E}_{s \sim U(\{1, ..., t\}) }\mathbb{E}_{\pmb{x} \sim p^{*}(\pmb{x})}  \mathbb{E}_{\pmb{\mu}_{s-1} | \pmb{x}}  \norm{\pmb{\epsilon} - \pmb{\epsilon}_{\pmb{\varphi}}(\pmb{\mu}_{s-1},s) }_{2}^{2} 
\end{align}
When $\alpha_{s}$ is a positive increasing sequence, this weighting scheme is very similar to the original PMM objective in the sense that timesteps containing less information about $\pmb{x}$ recieve less weight. Empirically, we found this modification of the loss to simplify implementation and to have a negligible effect on sample quality. To see why it simplifies implementation, consider the noise schedule used in our experiments 
\begin{align}
    N &:= \text{Number of posterior mean updates (i.e. $3000$ or $5000$).} \\ 
    \alpha_{s} &= \frac{13}{250}e^{13 t_{s}} \delta t \label{eqn:exp_weight} \\
    \delta t &= 1/N \\ 
    t_{s} &= s/N \\ 
    \sum_{k=1}^{s} \alpha_{s} &\approx \int_{0}^{s} \frac{13}{250} e^{13t} dt = \frac{1}{250} (e^{13t} - 1) \label{eqn:sum_int}
\end{align}
Thus, we have that 
\begin{align}
    \frac{\alpha_{s}}{13 \delta t}  \approx \sum_{k=1}^{s} \alpha_{s} + \frac{1}{250} \approx \sum_{k=1}^{s}\alpha_{s}. 
\end{align}
Now, note that 
\begin{align}
    \eta \nabla_{g_{\pmb{\varphi}}} \mathcal{L}_{\text{PMM}} &= \eta \alpha_{s}( \pmb{x} - g_{\pmb{\varphi}}) = \eta \frac{13}{250}e^{13t}\delta t (\pmb{x} - g_{\pmb{\varphi}}) \\
     \nabla_{g_{\pmb{\varphi}}} \mathcal{L}_{\text{PMM reweighted}} &= \sum_{k=1}^{s}\alpha_{s} (\pmb{x} - g_{\pmb{\varphi}}) \approx \frac{1}{250} e^{13 t} (\pmb{x} - g_{\pmb{\varphi}}) 
\end{align}
Thus, choosing $\eta = (13 \delta t)^{-1}$, we see that $\eta \nabla_{g_{\pmb{\varphi}}} \mathcal{L}_{\text{PMM}} \approx \nabla_{g_{\pmb{\varphi}}} \mathcal{L}_{\text{PMM reweighted}}$. However, note that $\nabla_{g_{\pmb{\varphi}}} \mathcal{L}_{\text{PMM}}$ depends on the resolution $\delta t$ used to approximate the sum in equation (\ref{eqn:sum_int}) with an integral. The reason this simplifies the implementation of PMMs is that the gradient of the objective $\nabla \mathcal{L}_{\text{PMM}}(\pmb{\varphi})$ using the exponential weighting of equation (\ref{eqn:exp_weight})  depends on $\delta t$. This means that the training dynamics depend on the number of posterior mean updates. However, replacing the factor of $\alpha_{s}$ with the corresponding sum $\sum_{k=1}^{s} \alpha_{k}$, gets rid of this dependency. This makes the training dynamics less sensitive to the choice of the learning rate, which is why we use the reweighted version of the PMM loss in our experiments. However, we want to emphasize that this choice implicitly corresponds to a simple adjustment of the learning rate. 
\newpage 
\subsubsection{Connection to SDEs and Diffusion}
\begin{theorem}\label{thm:online_bayes_diffusion_app}
    \textbf{(Online Bayesian Inference as a Diffusion Process)} \\
    Consider the update rule for the posterior mean \(\mu_t\) given by \eqref{eqn:online_bayesapp}:
    \begin{align}
        \pmb{\mu}_{t} &= \frac{b + \sum_{s=1}^{t-1}\alpha_{s}}{b + \sum_{s=1}^{t}\alpha_{s}} \pmb{\mu}_{t-1} + \frac{\alpha_{t}}{b + \sum_{s=1}^{t}\alpha_{s}} \pmb{y}_{t}.\label{eqn:online_bayesapp}
    \end{align}
    Let \(f : [0,1] \to \mathbb{R}^{+}\) and consider \(0 = t_{1} < t_{2} < \ldots < t_{T} = 1\) a partition of the unit interval. Moreover, define the sequence \(\alpha_{1}, \ldots, \alpha_{T}\) from \eqref{eqn:bayes2app} by \(\alpha_{s} = f(t_{s}) \delta t_{s}\). In the limit as \(T \to \infty\) and \(\delta t_{s} \to 0\), the discrete updates of \eqref{eqn:online_bayesapp} converge to a diffusion process defined by the following Stochastic Differential Equation (SDE):
    \begin{align}
        d\pmb{\mu}(t) &= f(t)\frac{(\pmb{x} - \pmb{\mu}(t))}{b + \int_{0}^{t} f(\tau) \, d\tau} dt + \frac{\sqrt{f(t)}}{b + \int_{0}^{t} f(\tau) \, d\tau} d\pmb{W}_{t}, \quad 0 \leq t \leq 1 \\ 
        \pmb x &\sim p^{*}(\pmb{x}) \\
        \mu(0) &= 0 
    \end{align}
\end{theorem}

\begin{proof}
    Let $\pmb\mu(t_{s}) \equiv \pmb\mu_{s}$. Using this notation, the posterior update rule is given by 
    \begin{align}
        \pmb\mu(t_{s}) &= \frac{b + \sum_{s'=1}^{s-1}f(t_{s'}) \delta t_{s'}}{b + \sum_{s'=1}^{s}f(t_{s'}) \delta t_{s'}}\pmb\mu(t_{s-1}) + \frac{f(t_{s}) \delta t_{s}}{b + \sum_{s'=1}^{s}f(t_{s'}) \delta t_{s'}}\pmb{y}_{s} 
    \end{align}
    Substituting \(\pmb{y}_s \sim \mathcal{N}(\pmb{x}, (f(t_s) \delta t_s)^{-1}I)\), we have:
    \begin{align}
        \pmb\mu(t_{s}) &= \frac{b + \sum_{s'=1}^{s-1}f(t_{s'}) \delta t_{s'}}{b + \sum_{s'=1}^{s}f(t_{s'}) \delta t_{s'}}\pmb\mu(t_{s-1}) + \frac{f(t_{s}) \delta t_{s}}{b + \sum_{s'=1}^{s}f(t_{s'}) \delta t_{s'}}\mathcal{N}(\pmb{x}, (f(t_{s}) \delta t_{s})^{-1}I) 
    \end{align}
    Rearranging terms:
    \begin{align}
        \pmb\mu(t_{s}) - \pmb\mu(t_{s-1}) &= \frac{ - f(t_{s}) \delta t_{s}}{b + \sum_{s'=1}^{s}f(t_{s'}) \delta t_{s'}} \pmb\mu(t_{s-1}) + \left( \frac{f(t_{s}) \delta t_{s}\pmb{x} + \sqrt{f(t_{s})} \epsilon_{t_{s}} \sqrt{\delta t_{s}}}{b + \sum_{s'=1}^{s}f(t_{s'}) \delta t_{s'}}  \right) \\ 
        &= \frac{(\pmb{x} - \pmb\mu(t_{s-1})) f(t_{s})}{b + \sum_{s'=1}^{s}f(t_{s'}) \delta t_{s'}}\delta t_{s} + \frac{\sqrt{f(t_{s})}}{b + \sum_{s'=1}^{s}f(t_{s'}) \delta t_{s'}}\pmb \epsilon_{t_{s}} \sqrt{\delta t_{s}} 
    \end{align}
    Where $\pmb \epsilon_{t_{s}}$ is an independent standard normal random variable. Taking the continuum limit $\delta t_{s} \to 0$ and $T \to \infty$, this process converges to the diffusion process 
    \begin{align}
        d\pmb\mu(t) = \frac{(\pmb{x} - \pmb\mu(t))f(t)}{b + \int_{0}^{t} f(\tau) d\tau} dt + \frac{\sqrt{f(t)}}{b + \int_{0}^{t} f(\tau) d\tau}d\pmb W_{t}
    \end{align}
\end{proof}

\newpage 

\subsection{Diriclet-Categorical Model}\label{apx:dircat-pmm-obj}
\subsubsection{Posterior Mean Updates} \label{apx:dir-cat-pmm-update}
The posterior distribution of the Dirichlet-Categorical model is given by
\begin{align}
\pmb{x}_c | \pmb{y}_{1:t, c}, w_{1:t, c} &\sim \text{Dirichlet}_{V}\left( \frac{\mathbbm{1}_{V}}{K} + \sum_{t'=1}^{t} \pmb{y}_{t'c} \right)
\end{align}
This closed-form posterior allows us to update of our beliefs about the true tokens as we observe more noisy versions.
The posterior mean is given by:
\begin{align}
\pmb{\mu}_{tc} &= \mathbb{E}\left(\pmb{x}_c | \pmb{y}_{1:t, c}, w_{1:t, c} \right) \\
&= \frac{\frac{\mathbbm{1}_{V}}{K} +\sum_{t'=1}^{t} \pmb{y}_{t'c} }{\frac{V}{K} +\sum_{t'=1}^{t}\sum_{d=1}^{V}y_{t'c}^{(d)} }
\end{align}
To simplify notation, let $N_{tc} = \sum_{t'=1}^{t}\sum_{d=1}^{V}y_{t'c}^{(d)}$ be the number of non-mask tokens observed up to timestep $t$. Then, we can express the distribution of $\pmb{\mu}_{tc} | \pmb{\mu}_{t-1,c}, \pmb{x}_c$ as:
\begin{align}
    \pmb{\mu}_{tc} | \pmb{\mu}_{t-1,c}, \pmb{x} &= \begin{cases}
        \frac{\frac{\mathbbm{1}_{V}}{K} + N_{t-1,c}}{\frac{V}{K} + N_{t-1,c} + 1} \pmb{\mu}_{t-1,c} + \frac{\pmb{\tilde{y}}_{tc}}{\frac{V}{K} + N_{t-1,c} + 1} \quad \text{with probability } w_{tc} \\
        \pmb{\mu}_{t-1,c} \quad \text{with probability } 1 - w_{tc} 
    \end{cases} \label{text1} \\ 
    \pmb{\tilde{y}}_{tc} &\sim \text{Cat}_{V}(x_{c}^{(1)}, ..., x_{c}^{(V)}) \label{text2}
\end{align}
To use a non-Informative Dirichlet prior, we take $K \to \infty$, and results in the following update equations
\begin{align}
    \pmb{\mu}_{sc} | (\pmb{\mu}_{s-1,c} = \frac{\mathbbm{1}_{V}}{V}, \pmb{x})     &\overset{d}{=} 
    \begin{cases}
         \text{Cat}(\pmb{x}_{c}) \quad \text{w/prob } \alpha_{tc} \\
       \frac{\mathbbm{1}_{V}}{V} \quad \text{w/prob } 1 - \alpha_{tc} 
    \end{cases}  \\ 
    \pmb{\mu}_{tc} | (\pmb{\mu}_{t-1,c} \neq \frac{\mathbbm{1}_{V}}{V} , \pmb{x}) &= \pmb{\mu}_{t-1,c}.
\end{align}

By applying this update, the posterior mean simplifies substantially, becoming either (a) a uniform distribution over all tokens or (b) a one-hot encoded vector. This means that the mean vector is no longer a dense vector, which leads to massive gains in computational efficiency. 
\newpage 
\subsubsection{Consistency of Dirichlet Categorical PMM}
\begin{theorem}
    (Concentration of posterior mean: Dirichlet-Categorical) Consider a Dirichlet Categorical PMM for text of context length $C$. Suppose that for all $c \in \{ 1, ..., C \}$, we have $\prod_{s=1}^{t} (1 - \alpha_{tc}) \to 0$. Then, the Dirichlet-Categorical model is consistent under a non-informative Dirichlet prior. 
\end{theorem}
\begin{proof}
    The probability that $\pmb{\mu}_{tc} = \frac{\mathbbm{1}_{V}}{V}$ is given by 
    \begin{align}
        \mathbb{P}\left( \pmb{\mu}_{tc} = \frac{\mathbbm{1}_{V}}{V} \right) &= \prod_{s=1}^{t}(1-\alpha_{tc}) \to 0 
    \end{align}
    Under a non-informative prior, we also know that $\pmb{\mu}_{tc} \neq \frac{\mathbbm{1}_{V}}{V} \iff \pmb{\mu}_{tc} = \pmb{x}_{tc}$. Thus, 
    \begin{align}
        \mathbb{P}\left(\pmb{\mu}_{t} \neq \pmb{x} \right) &= \mathbb{P}\left(\exists c : \pmb{\mu}_{tc} = \frac{\mathbbm{1}_{V}}{V}\right) \\ 
        &\leq \sum_{c} \prod_{s=1}^{t}(1-\alpha_{tc}) \to 0 
    \end{align}
    This implies that $\pmb{\mu}_{t} \to \pmb{x}$ in probability as $t \to \infty$. 
\end{proof}

\newpage 
\subsubsection{PMM Objective using a Non-Informative Prior} \label{sec:non_info_loss}
We calculate the PMM objective under a non-informative Dirichlet prior using the distribution $q_{\pmb{\varphi}}(\pmb{\mu}_{1}, ..., \pmb{\mu}_{t})$ of equations (\ref{eqn:apx_online_bayes_text1}) and (\ref{eqn:apx_online_bayes_text2}) 
\begin{align}
    \mathcal{L}_{\text{DirCat}}(\pmb{\varphi})  &\propto - \sum_{t=1}^{T} \mathbb{E}_{x, \pmb{\mu}_{t-1}} \mathbb{E}_{\pmb{\mu}_{t} | \pmb{\mu}_{t-1}} \log{q_{\pmb{\varphi}}(\pmb{\mu}_{t} | \pmb{\mu}_{t-1})} \\ 
    &=  - \sum_{t=1}^{T} \mathbb{E}_{x, \pmb{\mu}_{t-1}} \mathbb{E}_{\pmb{\mu}_{t} | \pmb{\mu}_{t-1}} \sum_{c} \log{q_{\pmb{\varphi}}(\pmb{\mu}_{tc} | \pmb{\mu}_{t-1})} \\
    &= - \sum_{tc}\mathbb{E}_{x, \pmb{\mu}_{t-1}} \mathbb{E}_{\pmb{\mu}_{tc} | \pmb{\mu}_{t-1}}  \log{q_{\pmb{\varphi}}(\pmb{\mu}_{tc} | \pmb{\mu}_{t-1})} \\
    &= - \sum_{tc}\mathbb{E}_{x, \pmb{\mu}_{t-1, - c}, \pmb{\mu}_{t-1,c}} \mathbb{E}_{\pmb{\mu}_{tc} | \pmb{\mu}_{t-1,c}}  \log{q_{\pmb{\varphi}}(\pmb{\mu}_{tc} | \pmb{\mu}_{t-1})}\\
    &= - \sum_{tc}\mathbb{E}_{x, \pmb{\mu}_{t-1, - c}, \pmb{\mu}_{t-1,c}} 1\left(\pmb{\mu}_{t-1,c} = 1/V \right)\mathbb{E}_{\pmb{\mu}_{tc} | \pmb{\mu}_{t-1,c} = 1/V}  \log{q_{\pmb{\varphi}}(\pmb{\mu}_{tc} | \pmb{\mu}_{t-1})} \\
    &+ 1(\pmb{\mu}_{t-1,c} \neq 1/V) \mathbb{E}_{\pmb{\mu}_{tc} | \pmb{\mu}_{t-1,c} \neq 1/V}  \log{q_{\pmb{\varphi}}(\pmb{\mu}_{tc} | \pmb{\mu}_{t-1})} \\
    &= - \sum_{tc}\mathbb{E}_{x, \pmb{\mu}_{t-1, - c}, \pmb{\mu}_{t-1,c}} 1\left(\pmb{\mu}_{t-1,c} = 1/V \right)\mathbb{E}_{\pmb{\mu}_{tc} | \pmb{\mu}_{t-1,c} = 1/V}  \log{q_{\pmb{\varphi}}(\pmb{\mu}_{tc} | \pmb{\mu}_{t-1})} \\ 
    &= - \sum_{tc} \mathbb{E}_{x, \pmb{\mu}_{t-1, - c}, \pmb{\mu}_{t-1,c}} 1\left(\pmb{\mu}_{t-1,c} = 1/V \right) \alpha_{tc} \log{f_{\pmb{\varphi}}^{(x_{c})}(\pmb{\mu}_{t-1})_{c}}
\end{align}

\newpage
\subsubsection{Continuous-time Objective for the Dirichlet-Categorical PMM} \label{apx:dir-cat-cts}
To obtain a continuous-time formulation of the Dirichlet Categorical PMM model we consider a partition of the unit interval $0 = t_{0} <  t_{1} < ... < t_{T} = 1$ and define $\alpha_{sc} = f_{c}(t_{s})\delta t_{s}$ where $f$ is a positive function defined on the unit interval $[0,1]$. Using this notation we index the posterior mean using this partition $\pmb{\mu}_{t_{s}}$. With a non-informative prior, the continuous time formulation of the Posterior Mean Matching Objective is given by 

\begin{align}
    \mathcal{L}_{\text{DirCat}}^{(\infty)}(\pmb{\varphi}) &\propto - \sum_{s=1}^{T} \mathbb{E}_{x, \pmb{\mu}_{t_{s}-\delta t_{s}}} \mathbb{E}_{\pmb{\mu}_{t_{s}} | \pmb{\mu}_{t_{s} - \delta t_{s} }} \sum_{c} \log{q_{\pmb{\varphi}}(\pmb{\mu}_{t_{s}, c} | \pmb{\mu}_{t-\delta t_{s}}) } \\ 
    &= - \sum_{s=1}^{T} \mathbb{E}_{x, \pmb{\mu}_{t_{s}-\delta t_{s}}}\sum_{c} \mathbb{E}_{\pmb{\mu}_{t_{s}, c} | \pmb{\mu}_{t_{s} - \delta t_{s}, c = 1/V}} 1( \pmb{\mu}_{t_{s} - \delta t_{s}, c } = 1/V) \log{q_{\varphi}(\pmb{\mu}_{t_{s},c} | \pmb{\mu}_{t_{s} - \delta t_{s}})} \\
    & +  \mathbb{E}_{\pmb{\mu}_{t_{s},c} | \pmb{\mu}_{t_{s} - \delta t_{s}, c \neq 1/V}}1( \pmb{\mu}_{t_{s} - \delta t_{s}, c } \neq 1/V) \log{q_{\varphi}(\pmb{\mu}_{t_{s},c} | \pmb{\mu}_{t_{s} - \delta t_{s}})} \\
    &= - \sum_{s=1}^{T} \mathbb{E}_{x, \pmb{\mu}_{t_{s}-\delta t_{s}}}  \sum_{c} \mathbb{E}_{\pmb{\mu}_{t_{s}, c} | \pmb{\mu}_{t_{s} - \delta t_{s}, c = 1/V}} 1( \pmb{\mu}_{t_{s} - \delta t_{s}, c } = 1/V) \log{q_{\varphi}(\pmb{\mu}_{t_{s},c} | \pmb{\mu}_{t_{s} - \delta t_{s}})} \\ 
    &= - \sum_{s=1}^{T} \mathbb{E}_{x, \pmb{\mu}_{t_{s}-\delta t_{s}}} \sum_{c} 1( \pmb{\mu}_{t_{s} - \delta t_{s}, c } = 1/V) f(t_{s}) \delta t_{s} \log{\hat{x}_{\pmb{\varphi}}^{x_{c}}(\pmb{\mu}_{t_{s}-\delta t_{s} })_{c}} \\
    &\to - \int_{0}^{1} \mathbb{E}_{x, \pmb{\mu}_{s^{-}}} \sum_{c} 1(\pmb{\mu}_{s^{-},c} = 1/V) f(s) \log{\hat{x}_{\pmb{\varphi}}^{(x_{c})}(\pmb{\mu}_{s^{-}})} ds \label{eqn:cont-dir-cat}
\end{align}
To obtain a Monte Carlo estimator of the PMM objective, we need to obtain samples from $\pmb{\mu}_{s-,c}$ at an arbitrary time-step. This is possible by noting that the event $\pmb{\mu}_{s-,c} = 1/V$ has probability 
\begin{align}
    \mathbb{P}(\pmb{\mu}_{t_{s},c} = 1/V) &= 1 - \prod_{s'=1}^{s}(1 - \alpha_{s'c}) \\
    &= 1 - \prod_{s'=1}^{s} ( 1 - f_{c}(t_{s'}) \delta t_{s'}) \\
    &\approx 1 - \prod_{s'=1}^{s} \exp{- f_{c}(t_{s'}) \delta t_{s'} } \\
    &= 1 - \exp{ - \sum_{s'=1}^{s} f_{c}(t_{s'}) \delta t_{s'}} \overset{\delta t \to 0 }{\to} 1 - \exp{- \int_{0}^{s} f(\tau) d\tau} 
\end{align}
Otherwise, $\pmb{\mu}_{s-,c} = \pmb{x}_{c}$ . 

\newpage 
\subsection{Details of the Gamma-Poisson PMM}
\subsubsection{Online Bayesian Updates}\label{apx:poisson-online}
\paragraph{Notation} If $\pmb{x} \in \mathbb{N}^{d}$, we write $\pmb{x} \sim \text{Poisson}(\lambda)$ to mean that each coordinate of $\pmb{x}$ is sampled independently from a Poisson distribution with rate parameter $\lambda$. Similarly, if $\pmb{x} \in \mathbb{R}^{d}$ we write $\pmb x \sim \Gamma(\beta_{1}, \beta_{2})$ to mean that each coordinate of the vector $\pmb x$ is sampled from a Gamma distribution with shape and rate parameters $\beta_{1}$ and $\beta_{2}$, respectively. 

For the Gamma-Poisson Bayesian model:
\begin{align}
     \pmb{x} &\sim \Gamma(\beta_{1}, \beta_{2}) \\
     \pmb{y}_{s} | \pmb{x} &\sim \text{Poisson}( \alpha_{t}\pmb{x}) , \quad \forall s \in \{1, ..., t\} 
\end{align}
the posterior distribution under this model is given by 
\begin{align}
    p( \pmb{x} | \pmb{y}_1, \dots, \pmb{y}_t) = 
    \text{Poisson}\left(\pmb{x}; \frac{\beta_{1} + \sum_{s=1}^{t} \pmb y_{t}}{\beta_{2} + \sum_{s=1}^{t}\alpha_{t}} \right)
\end{align}
From this we can read off the posterior mean 
\begin{align}
        \pmb{\mu}_{t} &\overset{d}{=} \frac{\beta_{1} + \sum_{s=1}^{t} \pmb y_{t}}{\beta_{2} + \sum_{s=1}^{t}\alpha_{t}}
\end{align}
Note that we can rewrite the above expression as:
\begin{align}
    \pmb{\mu}_t &= \frac{\beta_1 + \sum_{i=1}^{t-1} \pmb{y}_i}{\beta_2 + \sum_{i=1}^t \alpha_i} + \frac{1}{\beta_2 + \sum_{i=1}^t \alpha_i}\pmb{y}_t \\
    &= \frac{\beta_2 + \sum_{i=1}^{t-1} \alpha_i}{\beta_2 + \sum_{i=1}^t \alpha_i}\pmb{\mu}_{t-1} + \frac{1}{\beta_2 + \sum_{i=1}^t \alpha_i}\pmb{y}_t,
\end{align}
giving us the following expressions for the online Bayesian update for the Gamma-Possion model:
\begin{align}
    \boxed{\pmb{\mu}_t = \frac{A_t}{A_{t-1}}\pmb{\mu}_{t-1} + A_t \pmb{y}_t, }
\end{align}
where $A_t = \left(\beta_2 + \sum_{i=1}^t \alpha_i \right)^{-1}$.

\newpage 
\subsubsection{Consistency of Gamma-Poisson Model}
\begin{theorem}
    Let $ \pmb{x}^* > 0 $ be a sample drawn from some underlying true distribution $\pmb x^* \sim p^*(\pmb x)$, and let $ \{ \pmb y_1, \pmb y_2, \ldots, \pmb y_t \} $ be observations generated according to the following noise model:
    \begin{align}
         \pmb{y}_{t} | x^* \sim \text{Pois}(\alpha_{t} \pmb x^*)
    \end{align}
    where \( \alpha_s \) is a known, positive, increasing sequence satisfying \( \lim_{t \to \infty} \sum_{s=1}^{t} \alpha_s = O(t^{1 + \eta}) \to \infty \) with $\eta$ being an arbitrarily small positive number. Then, the posterior mean \( \pmb{\mu}_t \) of the following Bayesian model, 
    \begin{align}
        \pmb{x} &\sim \Gamma(\beta_{1}, \beta_{2}), \label{eqn:gp-prior}\\
        \pmb{y}_{t} | &\pmb{x} \sim \text{Pois}(\alpha_{t}\pmb{x}), \label{eqn:gp-likelihood}
    \end{align} 
    is consistent, namely:
\begin{align}
    \lim_{t \to \infty} \pmb{\mu}_t = \pmb x^*, \quad \text{almost surely}.
\end{align}
\end{theorem}
\begin{proof}
    We give a proof for the one dimensional case with the understanding that it extends to the multidimensional case by applying the same argument to each coordinate individually. Under the Bayesian model (Equations \ref{eqn:gp-prior} and \ref{eqn:gp-likelihood}), we have that the posterior
    \begin{align}
        x | y_{1:t} \sim  \Gamma\left(\beta_{1} + \sum_{s=1}^{t} y_s, \beta_{2} + \sum_{s=1}^{t}\alpha_{s} \right)
    \end{align}
    The posterior mean under this Bayesian Model is given by 
    \begin{align}
    \mu_{t} &\equiv \mathbb{E}(x | y_{1:t}) \\
    &= \frac{\beta_{1} + \sum_{s=1}^{t} y_s}{\beta_{2} + \sum_{s=1}^{t}\alpha_{s}} =  \frac{\beta_{1} + \sum_{s=1}^{t} \text{Pois}(\alpha_{s} x)}{\beta_{2} + \sum_{s=1}^{t}\alpha_{s}}
    \end{align}
    Where the last equality follows by equation (\ref{eqn:gp-likelihood}). Now we have that:
    \begin{align*}
        \mathbb{E}( \mu_{t} )  &= \mathbb{E}\left( \frac{\beta_{1} + \sum_{s=1}^{t} \text{Pois}(\alpha_{s} x)}{\beta_{2} + \sum_{s=1}^{t}\alpha_{s}} \right) = 
        \frac{\beta_{1} +\mathbb{E}\left(\sum_{s=1}^{t} \text{Pois}(\alpha_{s} x^*) \right)}{\beta_{2} + \sum_{s=1}^{t}\alpha_{s}} \\
        &= \frac{\beta_{1} +\sum_{s=1}^{t} \alpha_{s} x^*}{\beta_{2} + \sum_{s=1}^{t}\alpha_{s}} = \frac{\beta_{1}}{\beta_{2} + \sum_{s=1}^{t}\alpha_{s}} + \frac{\sum_{s=1}^{t} \alpha_{s}}{\beta_{2} + \sum_{s=1}^{t}\alpha_{s}} x^*
    \end{align*}
    Assuming that $\sum_{s=1}^{t} \alpha_{s} \to \infty$, we see that $\frac{\beta_{1}}{\beta_{2} + \sum_{s=1}^{t}\alpha_{s}} \to 0$ and $\frac{\sum_{s=1}^{t} \alpha_{s}}{\beta_{2} + \sum_{s=1}^{t}\alpha_{s}} \to 1 $, showing  $\mathbb{E}( \mu_{t} ) \to x^*$. Next,
    \begin{align*}
        \text{Var}( \mu_{t} )  &= \frac{1}{(\beta_{2} + \sum_{s=1}^{t}\alpha_{s})^2}  \text{Var}\left( \beta_{1} + \sum_{s=1}^{t} \text{Pois}(\alpha_{s} x^*) \right) \\
        &= \frac{\sum_{s=1}^{t} \alpha_{s}}{(\beta_{2} + \sum_{s=1}^{t}\alpha_{s})^2} x^*
    \end{align*}
    
    Noting that $\frac{\sum_{s=1}^{t} \alpha_{s}}{\left(\beta_{2} + \sum_{s=1}^{t}\alpha_{s}\right)^2} \to 0$, we have that $\text{Var}(\mu_{t}) \to 0$. Thus showing $\mathbb{E}(\mu_{t}) \to x^*$ and $\text{Var}(\mu_{t}) \to 0$, implies that 
    $\mu_{t} \to x^*$ in probability.

    To obtain almost sure convergence, let $a_{t} = \frac{\beta_{1}}{\beta_{2} + \sum_{s=1}^{t}\alpha_{s}}  $ and $b_{t} = \frac{\sum_{s=1}^{t} \alpha_{s}}{\beta_{2} + \sum_{s=1}^{t}\alpha_{s}}$ and consider the event $A_{t} = \{ \abs{\pmb{\mu}_{t} - (b_{t} + a_{t} \pmb{x})} > \epsilon \}$. As in the Normal-Normal model, an application of Chebyshev's inequality to these events results in $\sum_{s=1}^{\infty} \mathbb{P}(A_{t}) \leq \epsilon^{-2} \sum_{s=1}^{\infty} \text{Var}(\pmb{\mu}_{t}) \lesssim  \epsilon^{-2} \sum_{t=1}^{\infty} (\sum_{s=1}^{t} \alpha_{s})^{-1} < \infty$ (for this to happen $(\sum_{s=1}^{t} \alpha_{s})^{-1} = \Omega(t^{-(1 +\eta)})$ where $\eta$ is an arbitrarily small number). Now, from the Borel-Cantelli Lemma, it follows that at most finitely many events from the collection $\{A_{t}\}_{t=1}^{\infty}$ can occur. In other words, with probability $1$, we have that for all $\epsilon > 0$ there is a $T \in \mathbb{N}$ such that for all $s > T$ we have $\abs{\pmb{\mu}_{s} - (b_{s} + a_{s}\pmb{x})} < \epsilon $. Therefore, $\lim_{s\to\infty} \abs{\pmb{\mu}_{s} - (b_{s} + a_{s}\pmb{x})} = \lim_{s\to\infty} \abs{\pmb{\mu}_{s}  - \pmb{x}} < \epsilon $. This means that with probability $1$,  $\pmb{\mu}_{s} \to \pmb{x}$, as desired.
\end{proof}

\newpage 

\subsubsection{Gamma-Poisson PMM Objective}
The Gamma-Poisson PMM objective follows by the following straightforward calculation
\begin{align}
    \mathcal{L}(\pmb \varphi) &\propto - \mathbb{E}_{\pmb x, \pmb \mu_{1}, ..., \pmb \mu_{t}} \log{q_{\pmb \varphi}(\pmb \mu_{1}, ..., \pmb \mu_{t})} \\
    &= - \sum_{t} \mathbb{E}_{\pmb x, \pmb \mu_{t}, \pmb \mu_{t-1}} \log{q_{\pmb \varphi}(\pmb \mu_{t} | \pmb \mu_{t-1})}\\
    &\propto - \sum_{t} \mathbb{E}_{\pmb x, \pmb \mu_{t}, \pmb \mu_{t-1}} \log{\prod_{n}\text{Pois}\left( \mu_{tn} ; f_{\pmb \varphi}(\pmb \mu_{t-1}, t)_{n} \right)} \\ 
    &= - \sum_{t,n} \mathbb{E}_{\pmb x, \pmb \mu_{t}, \pmb \mu_{t-1}} \mu_{tn} \log{f_{\pmb \varphi}(\pmb \mu_{t-1}, t)_{n}} - f_{\pmb \varphi}(\pmb \mu_{t-1}, t)_{n}
\end{align}
\newpage 
\subsubsection{Connection between Gamma-Poisson PMMs and SDEs}
\begin{lemma}\label{lem:ito_pois}
    (Ito's Lemma for Poisson Processes) Let $N_{t}$ be a non-homogenous Poisson Process with rate function $\lambda(t)$. Then if we let $f(N_{t}, t)$, it follows that $df_{t}$, satisfies the Stochastic Differential Equation 
    \begin{align}
        df(N_{t},t) &= (f(N_{t} + 1, t) - f(N_{t}, t))dN_{t} + \pdv{f(N_{t},t)}{t} dt 
    \end{align}
\end{lemma}
Since this version of Ito's lemma isn't as widespread as it's counterpart for Brownian Motion, we provide a proof below: 
\begin{proof}
    The infinitesimal characterization of the Poisson Process tells us that $dN_{t} = 0 $ with probability $1 - \lambda(t) dt + o(dt)$ and that $dN_{t} = 1$ with probability $\lambda(t)dt + o(dt)$. This means that 
    \begin{align}
        f(N_{t} + dN_{t}, t) &= \begin{cases}
            f(N_{t} + 1, t) \quad \text{with probability } 1 - \lambda(t) dt \\
            f(N_{t},t) \quad \text{with probability } \lambda(t) dt
        \end{cases}
    \end{align}
    As a result, 
    \begin{align}
        df(N_{t}, t) &= f(N_{t} + dN_{t}, t + dt) - f(N_{t}, t) \\
        &= f(N_{t} + dN_{t}, t) + \pdv{f(N_{t} + dN_{t}, t)}{t}dt - f(N_{t}, t) \\
        &= dN_{t}\left( f(N_{t} + 1, t) + \pdv{f(N_{t} + 1, t)}{t}dt - f(N_{t}, t) \right) + (1-dN_{t})\left( \pdv{f(N_{t}, t)}{t}dt \right)  \\ 
        &= \left(f(N_{t}+1,t) - f(N_{t},t) + \left(\pdv{f(N_{t} + 1, t)}{t} - \pdv{f(N_{t}, t)}{t}\right)dt \right)dN_{t} + \pdv{f(N_{t}, t)}{t} \\
        &= \left( f(N_{t} + 1, t) - f(N_{t}, t) \right) dN_{t} + \pdv{f(N_{t}, t)}{t}dt  + o(dt) 
    \end{align}
    Completing the proof. 
\end{proof}

\begin{theorem}\label{thm:gamma-poisson-sde-app}
    (Gamma-Poisson SDE) Consider the update rule of the posterior mean \(\mu_t\) for the Gamma-Poisson PMM shown in equation \ref{eqn:online-bayes-gamma-pois}. Let \(f : [0,1] \to \mathbb{R}^{+}\) and consider \(0 = t_{1} < t_{2} < \ldots < t_{T} = 1\) a partition of the unit interval. Moreover, define the sequence \(\alpha_{1}, \ldots, \alpha_{T}\) by \(\alpha_{s} = f(t_{s}) \delta t_{s}\). In the continuum limit \(T \to \infty\) and \(\max_{s}\delta t_{s} \to 0\), we have that the discrete updates of $\mu_{t}$ converge to a Merton jump process characterized by the following Stochastic Differential Equation (SDE):
    \begin{align}
        d\pmb \mu(t) &= \left(L'(t) + \frac{A'(t)}{A(t)}\left(\pmb \mu(t) - L(t)\right)\right) dt + A(t) d\pmb N(t)
    \end{align}
    Where $\pmb N(t)$ is a Cox Process with random base measure $\pmb x dt$ with $\pmb x \sim p^{*}(\pmb x)$. 
\end{theorem}
\begin{proof}
    We give a proof for the one dimensional case with the understanding that it extends to the multidimensional case by applying the same argument to each coordinate individually. Note that the posterior mean of the Gamma Poisson model in equation \ref{eqn:online-bayes-gamma-pois} is given by 
    \begin{align}
        \mu_{k} | x &= \frac{\alpha}{\beta + \sum_{s=1}^{k} \alpha_{s}} + \frac{\sum_{s=1}^{k} \text{Pois}(x \delta t_{s}) }{\beta + \sum_{s=1}^{k}\alpha_{s}}
    \end{align}
    Fixing a partition of the unit interval $0 = t_{1} < t_{2} < ... < t_{T} = 1$ and reindexing $\alpha_{k} \equiv \alpha_{t_{k}} \equiv f(t_{k}) \delta t_{k}$ as a function of continuous time, it is not hard to see that the numerator of the second term converges to $N(t)$ ---a non-homogeneous Poisson Process with rate function $\lambda(t) = x dt$. Using this characterization, we view the posterior mean $\mu(t) = f(N_{t} , t) = L(t) + A(t)N(t)$ as a function of the non-homogeneous Poisson Process and apply Ito's Lemma for Poisson processes (Lemma \ref{lem:ito_pois}) to obtain 
    \begin{align}
        d\mu(N_{t} , t) &= \left(L(t) + A(t) (N(t) + 1) - (L(t) + A(t)N(t))\right)dN(t) + \left(A'(t) + L'(t) N(t)\right) dt \\
        &= (L'(t) + A'(t)N(t))dt + A(t) dN(t) 
    \end{align}
    Substituting $N(t) = (\mu(t) - L(t))/A(t)$, we obtain 
    \begin{align}
        d\mu(t) &= \left(L'(t) + \frac{A'(t)}{A(t)}\left(\mu(t) - L(t)\right)\right) dt + A(t) dN(t) 
    \end{align}
    Completing the proof 
\end{proof}
\newpage 
\subsection{Details of the InverseGamma-Gamma PMM}
\paragraph{Notation} If $\pmb{x} \in \mathbb{R}^{d}$ we write $\pmb x \sim \Gamma(\beta_{1}, \beta_{2})$ to mean that each coordinate of the vector $\pmb x$ is sampled from a Gamma distribution with shape and scale parameters $\beta_{1}$ and $\beta_{2}$, respectively. Similarly, we write $\pmb y \sim \Gamma(\alpha_{s}, \pmb x)$ to mean $y_{i} \overset{iid}{\sim} \Gamma(a, x_{i})$. \textbf{Note that unlike the Gamma-Poisson model, we use a different parametrization of the Gamma distribution throughout this section. We use the same parametrization for the inverse Gamma distribution.} \\ 

We consider the following noisy observation model 
\begin{align}
    \pmb x &\sim p^{*}(\pmb x) \label{eqn:gamma_data} \\ 
    \pmb y_{s} | \pmb x &\sim \Gamma(\alpha_s, \pmb x) \label{eqn:gamma-noise} \\
    \forall s &\in \{1, ..., t\} 
\end{align}
The corresponding Bayesian Model is given by 
\begin{align}
    \pmb x &\sim \text{Inv}\Gamma(\beta_{1}, \beta_{2}) \label{eqn:gamma-prior} \\ 
    \pmb y_{s} | \pmb x &\sim \Gamma(\alpha_{s}, \pmb x) \label{eqn:gamma-gamma-pmm} \\
    \forall s &\in \{1, ..., t\} 
\end{align}
The posterior distribution of this Bayesian model is given by 
\begin{align}
    \pmb x | \pmb y_{1:t} &\sim \text{Inv}\Gamma\left(\beta_{1} + \sum_{s=1}^{t}\alpha_{s}, \beta_{2} + \sum_{s=1}^{t}\pmb y_{t}\right) 
\end{align}
For $\beta_{1} > 1$ the posterior mean is well-defined and is given by 
\begin{align}
    \mathbb{E}(\pmb x | \pmb y_{1:t}) &= \frac{\beta_{2} + \sum_{s=1}^{t}\pmb y_{t} }{\beta_{1} + \sum_{s=1}^{t}\alpha_{s}} \label{eqn:inv-gamma-gamma-mean}
\end{align}
\subsubsection{Online Bayesian Update}
We rewrite the posterior mean of the InverseGamma-Gamma model to obtain the online Bayesian inference update rule. To simplify notation, let $A_{t} = \beta_{1} + \sum_{s=1}^{t} \alpha_s$, then 
\begin{align}
    \pmb \mu_{t} &= \mathbb{E}(\pmb x | \pmb y_{1:t}) \\
    &= \frac{\beta_{2} + \sum_{s=1}^{t-1}\pmb y_{s} + \pmb y_t }{A_t} \\
    &= \frac{A_{t-1}}{A_t} \pmb  \mu_{t-1} + \frac{\pmb y_t}{A_t}. 
\end{align}
\subsubsection{Consistency of the InverseGamma-Gamma Model}
\begin{theorem} \label{thm:norm-norm-invggamma}
(Concentration of posterior mean) Let \( \{ \pmb{y}_1, \ldots, \pmb{y}_t \} \) be observations generated according to equation (\ref{eqn:gamma-gamma-pmm}).  Suppose \( \alpha_t \) a known, positive, increasing sequence satisfying \( \lim_{t \to \infty}  \sum_{s=1}^{t} \alpha_s = \Omega(t^{1+\eta}) \) for all $\eta > 0$. Then, the posterior mean \( \pmb{\mu}_t \) of  the Bayesian model in equations (\ref{eqn:gamma-prior}) and (\ref{eqn:gamma-gamma-pmm}) is consistent, namely:
\begin{align}
    \lim_{t \to \infty} \pmb{\mu}_t = \pmb{x}, \quad \text{almost surely},
\end{align}
with respect to the joint distribution of \( (\pmb{x}, \pmb{y}_1, \pmb{y}_2, \ldots) \) in equations (\ref{eqn:gamma_data}) and (\ref{eqn:gamma-noise}). 
\end{theorem}
\begin{proof}
    The proof of consistency is identical to the consistency proofs of the Normal-Normal and Gamma-Poisson PMMs. We omit the proof for brevity.
\end{proof}
\subsubsection{InverseGamma-Gamma PMM Objective}
The InverseGamma-Gamma PMM objective choosing $q_{\pmb \varphi} (\pmb \mu_{1}, ..., \pmb{\mu}_{t})$ according to equations (\ref{eqn:noisy_model_online_approx}) and (\ref{eqn:online_bayes_approx})) is given by 
\begin{align}
    \mathcal{L}(\pmb \varphi) &\propto - \mathbb{E}_{\pmb x, \pmb \mu_{1}, ..., \pmb \mu_{t}} \log{q_{\pmb \varphi}(\pmb \mu_{1}, ..., \pmb \mu_{t})} \\
    &= - \sum_{t} \mathbb{E}_{\pmb x, \pmb \mu_{t}, \pmb \mu_{t-1}} \log{q_{\pmb \varphi}(\pmb \mu_{t} | \pmb \mu_{t-1})}\\
    &\propto - \sum_{t} \mathbb{E}_{\pmb x, \pmb \mu_{t}, \pmb \mu_{t-1}} \log{\prod_{n}\Gamma\left( \mu_{tn} ; \alpha_{t}, f_{\pmb \varphi}(\pmb \mu_{t-1}, t)_{n} \right)} \\ 
    &=  \sum_{t,n} \mathbb{E}_{\pmb x, \pmb \mu_{t}, \pmb \mu_{t-1}} \mu_{tn} \alpha_{t} \log{f_{\pmb \varphi}\left(\pmb \mu_{t-1}, t\right)_{n} } + \frac{\mu_{tn}}{f_{\pmb \varphi}\left(\pmb \mu_{t-1}, t\right)_{n}}
\end{align}
\newpage

\section{Additional Experiments and Experimental Details} \label{app:details}
\subsection{Experiments}
For all experiments, we estimate the Posterior Mean Matching objectives by using a batch of samples and Monte Carlo estimates of the expectations. The PMM objective is then minimized using Gradient Descent.  
\subsubsection{Neural Network Architectures}
\paragraph{Cifar-10} We use the Dhariwal UNet \citep{dhariwal2021diffusionmodelsbeatgans} implementation and architecture from \cite{Karras2022edm} and train it on the PMM objective with a batch size of 512 across 4 H100 GPUs for $637000$ for the Normal-Normal model and for $1200000$ iterations for the Gamma-Poisson Model. In both cases, we use the Adam Optimizer with a learning rate of $10^{-4}$ and no warmup. The samples were taken from an Exponential Moving Average of the Neural Network with a decay parameter of $0.9999$. 
\paragraph{AFHQ} We use the Dhariwal UNet \citep{dhariwal2021diffusionmodelsbeatgans} implementation and architecture from \cite{Karras2022edm} and train it on the PMM objective with a batch size of 624 across 4 H100 GPUs for $848000$ steps for the Normal-Normal model. In both cases, we use the Adam Optimizer with a learning rate of $10^{-4}$ and no warmup. The samples were taken from an Exponential Moving Average of the Neural Network with a decay parameter of $0.9999$. 

\paragraph{FFHQ} We use the Dhariwal UNet \citep{dhariwal2021diffusionmodelsbeatgans} implementation and architecture from \cite{Karras2022edm} and train it on the PMM objective with a batch size of 1248 across 8 H100 GPUs for four days for the Normal-Normal model. In both cases, we use the Adam Optimizer with a learning rate of $10^{-4}$ and no warmup. The samples were taken from an Exponential Moving Average of the Neural Network with a decay parameter of $0.9999$. 
\subsubsection{PMM Hyperparameters} \label{app:hyperparameters}
We report the choice of hyperparameters used for the PMM models trained in the experimental section in Table \ref{tab:PMM_hyp}. For the \texttt{text8} PMM model we report the BPC using the staircase schedule and we use the time-dependent schedule for the experiments on OpenWebText. 
\begin{table}[!ht]
\centering
\caption{Hyperparameter Choices with Equations}
\label{tab:PMM_hyp}
\begin{tabularx}{\textwidth}{l c X X}
\toprule
\textbf{Model} & \textbf{Number of Steps} & \textbf{Prior Parameters} & \textbf{Noise Schedule} \\
\midrule
Gamma-Poisson & 3000 & 
$\begin{aligned}
  \alpha &= 0.1 \\
  \gamma &= 2
\end{aligned}$ & 
$\begin{aligned}
    0 \leq t &\leq 1 \\ 
    t_{s} &= \frac{s}{3000},\ s \in \{0, \dots, 3000\} \\ 
    \alpha_{t_{s}} &= f(t_{s})\, dt \\ 
    f(t) &= \frac{13}{250} e^{\frac{t}{13}}
\end{aligned}$ \\
\midrule
Normal-Normal & 3000 & 
$b = 2$ & 
$\begin{aligned}
    0 \leq t &\leq 1 \\ 
    t_{s} &= \frac{s}{3000},\ s \in \{0, \dots, 3000\} \\ 
    \alpha_{t_{s}} &= f(t_{s})\, dt \\ 
    f(t) &= \frac{13}{250} e^{\frac{t}{13}}
\end{aligned}$  \\
\midrule
Dirichlet-Categorical & $\infty$ & 
$K \to \infty$ & 
$\begin{aligned}
    0 \leq t &\leq 1 \\ 
    \omega_{t c} &= f(t)\, dt \\ 
    f_{c}(t) &=  0.01 (1 + 2000t), \text{(time-dependent)} \\
    f_{c}(t) &= 3.5 (1 + 100 \bigg \lfloor \frac{t}{0.985} \bigg \rfloor ), \text{(staircase)} \\
    f_{c}(t) &= \sigma\left(\frac{t - c/C}{0.01}\right)/0.01, \text{(semi-AR)}
\end{aligned}$ \\
\bottomrule
\end{tabularx}
\end{table}

\subsection{Language Modeling}
\subsubsection{Neural Network Architectures}
The Dirichlet-Categorical PMM language model is based on the transformer architecture of the original GPT-2 model \citep{radford2019language}. The only difference is that we replace LayerNorm with Adaptive LayerNorm to condition on the timestep. This is similar to what is done with Diffusion Transformer \citep{Peebles2022DiT}. As far as network size goes, we set the network hyperparameters (number of transformer blocks, etc.) to match the ones from the SEDD and MDLM papers:
\begin{itemize}
    \item hidden\_size: 768
    \item cond\_dim: 128
    \item length: 1024
    \item n\_blocks: 12
    \item n\_heads: 12 
\end{itemize}
Like other diffusion language models, we do not tie the word embeddings at the input layer with the weights of the last linear transformation.

\subsubsection{Evaluation Details}
\paragraph{Text modeling} The baselines for the \texttt{Text8} dataset are taken from \cite{sedd}. It is possible to evaluate the number of nats per character of a PMM model using the following two facts: (a) If the step size is small enough then the probability that two tokens are unmasked at the same time-step is zero and (b) we can compute the expected negative log likelihood per character at the moment this character is unmasked. Since we compute the average log probability only over the unmasking events, the NPC (nats per character) computation reduces to
\begin{align}
    NPC &= - \sum_{tc} \mathbb{E}_{x, \pmb{\mu}_{t-1, - c}, \pmb{\mu}_{t-1,c}} \frac{1\left(\pmb{\mu}_{t-1,c} = 1/V \right) \alpha_{tc}}{\sum_{tc}1\left(\pmb{\mu}_{t-1,c} = 1/V \right) \alpha_{tc}} \log{f_{\pmb{\varphi}}^{(x_{c})}(\pmb{\mu}_{t-1})_{c}} 
\end{align}
As a sanity check, note that an autoregressive schedule corresponds to $1\left(\pmb{\mu}_{t-1,c} = 1/V \right) \alpha_{tc} = \delta_{tc}$ and that $\pmb \mu_{tc} = (\pmb x_{1}, ..., \pmb x_{c}, \text{mask}, ..., \text{mask})$ with probability $1$, therefore, the PMM model may be viewed as a conditional probability of the next token, given the previous tokens (aka. simply an autoregressive language model). Substituting this into the NPC formula we obtain 
\begin{align}
    NPC &= - \sum_{tc} \mathbb{E}_{x, \pmb{\mu}_{t-1, - c}, \pmb{\mu}_{t-1,c}} \frac{\delta_{tc}}{C} \log{f_{\pmb{\varphi}}^{(x_{c})}(\pmb{\mu}_{t-1})_{c}} \\
    &= -\sum_{c} \mathbb{E} \log{f_{\pmb{\varphi}}^{(x_{c})}(\pmb{\mu}_{c-1})_{c}} \\ 
    &=  -\frac{1}{C}\sum_{c} \mathbb{E} \log{f_{\pmb{\varphi}}^{(x_{c})}((\pmb x_{1}, ..., \pmb x_{c-1}, \text{mask}, ..., \text{mask}))_{c}}
\end{align}
Which, as expected, is just the nats per bit of an autoregressive language model. Converting this to bits is simply a matter of dividing by a factor of $\log{2}$.

We report the text8 results using a staircase schedule and the openwebtext evaluations using the time-dependent schedules shown in table \ref{tab:PMM_hyp}. The models were trained across 4 H100 GPUs for five days. This corresponded to roughly 980000 gradient steps for the text8 model and 940000 gradient steps for openwebtext. Both models where trained using the Adam optimizer with no learning rate warmup, a learning rate of $3\times 10^{-4}$ and $(\beta_{1}, \beta_{2}) = (0.9, 0.98)$. 
\paragraph{Unconditional generation}
In evaluating unconditional generation, we use the pre-trained checkpoints from Sahoo et al.\citep{mdlm_sahoo2024simple} to produce unconditional samples for the following models:
\begin{itemize}
    \item auto-regressive GPT-2 like transformer \citep{radford2019language},
    \item SEDD model\citep{sedd},
    \item and MDLM model\citep{mdlm}.
\end{itemize}
The checkpoints can be found at the MLDM\citep{mdlm_sahoo2024simple} GitHub\footnote{\url{https://github.com/kuleshov-group/mdlm}} repository, under the \textit{Checkpoints}\footnote{\url{https://github.com/kuleshov-group/mdlm?tab=readme-ov-file\#checkpoints}} section in the linked Google Drive. We used the following bash commands to generate a 1024 from each model:

\begin{minipage}[t]{0.32\textwidth}
\begin{lstlisting}[basicstyle=\tiny,breaklines=true]
# AR unconditional generation
CUDA_VISIBLE_DEVICES=0, python main.py \
  mode=sample_eval \
  eval.checkpoint_path=${checkpoint_path} 
  loader.batch_size=16 \
  loader.eval_batch_size=16 \
  sampling.num_sample_batches=64 \
  data=openwebtext-split \
  model=small-ar \
  parameterization=ar \
  backbone=ar \
  model.length=1024
\end{lstlisting}
\end{minipage}
\hfill
\begin{minipage}[t]{0.32\textwidth}
\begin{lstlisting}[basicstyle=\tiny,breaklines=true]
# SEDD unconditional generation
CUDA_VISIBLE_DEVICES=0, python main.py \
  mode=sample_eval \
  eval.checkpoint_path=${checkpoint_path} \
  loader.batch_size=16 \
  loader.eval_batch_size=16 \
  sampling.num_sample_batches=64 \
  sampling.predictor=analytic \
  sampling.steps=1000 \
  data=openwebtext-split \
  model=small \
  parameterization=sedd \
  backbone=dit \
  model.length=1024 \
  time_conditioning=True
\end{lstlisting}
\end{minipage}
\hfill
\begin{minipage}[t]{0.32\textwidth}
\begin{lstlisting}[basicstyle=\tiny,breaklines=true]
# MDLM unconditional generation
CUDA_VISIBLE_DEVICES=1, python main.py \
  mode=sample_eval \
  eval.checkpoint_path=${checkpoint_path} \
  loader.batch_size=16 \
  loader.eval_batch_size=16 \
  sampling.num_sample_batches=64 \
  sampling.predictor=ddpm_cache \
  sampling.steps=1000 \
  data=openwebtext-split \
  model=small \
  parameterization=subs \
  backbone=dit \
  model.length=1024 
\end{lstlisting}
\end{minipage}

Please make sure to set the \say{checkpoint\_path} variable in the bash script and link to the correct checkpoint. To install a correct environment from within one can run the MDLM codebase, create a new Python environment and install the following packages using \texttt{miniconda}:

\begin{lstlisting}[basicstyle=\tiny,breaklines=true]
conda install pytorch torchvision torchaudio pytorch-cuda=12.4 -c pytorch -c nvidia
conda install nvidia/label/cuda-12.4.0::cuda-toolkit
conda install lightning einops huggingface_hub transformers timm -c conda-forge
pip install rich omegaconf flash-attn
pip install hydra-core --upgrade
\end{lstlisting}

% \subsection{Additional Gamma-Poisson experiments}\label{app:gamma-pois-exp}

\newpage 
\section{Additional Figures} \label{app:figures}
\subsection{Convergence of the posterior mean}\label{app:convergence}
\begin{figure}[ht]
    \centering
    \subfloat[Normal PMM]{
        \begin{minipage}[b]{0.95\linewidth}
            \centering
            \includegraphics[width=\linewidth]{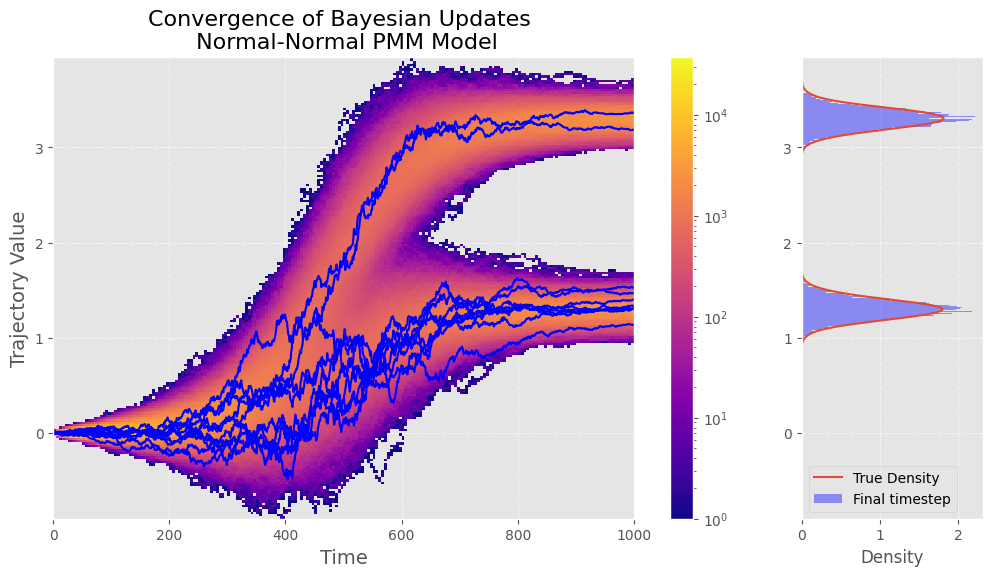}
        \end{minipage}
        \label{fig:normal_pmm_app}
    }
    \caption{Convergence of the posterior mean \( \pmb{\mu}_t \) to target samples \( \pmb{x} \sim p^*(\pmb{x}) \) as \( t \) increases for the Normal-Normal Posterior Mean Matching (PMM) model. 
    }
    \label{fig:pmean_convergence_app}
\end{figure}

\begin{figure}[ht]
    \centering
    \subfloat[Normal PMM]{
        \begin{minipage}[b]{0.95\linewidth}
            \centering
            \includegraphics[width=\linewidth]{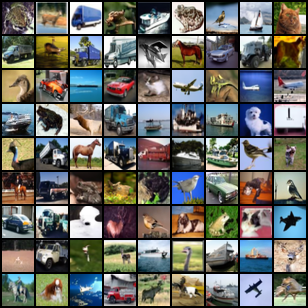}
        \end{minipage}
    }
    \caption{CIFAR 10 samples. FID = 2.46 with 500 NFEs
    }
    \label{fig:normal_pmm_fig_cifar}
\end{figure}

\begin{figure}[ht]
    \centering
    \subfloat[Normal PMM]{
        \begin{minipage}[b]{0.95\linewidth}
            \centering
            \includegraphics[width=\linewidth]{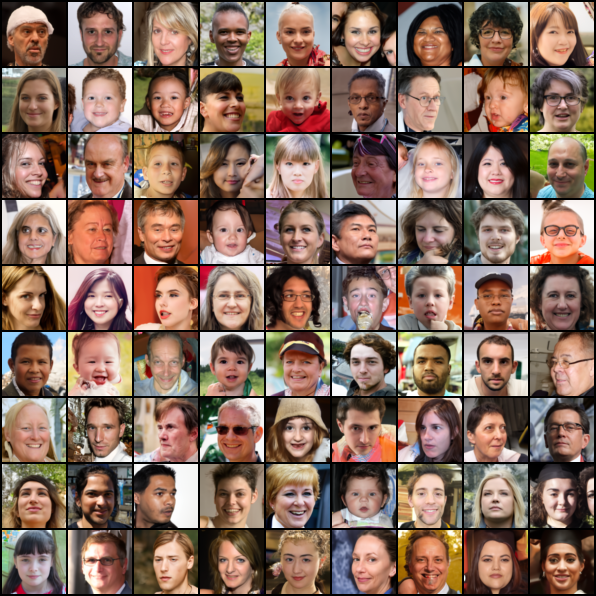}
        \end{minipage}
    }
    \caption{FFHQ samples. FID = 3.89 with 500 NFEs
    }
    \label{fig:normal_pmm_fig_FFHQ}
\end{figure}
\newpage 

\section{Sample text outputs from the Dir.-Cat. model}
\textcolor{red}{Disclaimer:} \textbf{Language models' outputs have not been safeguarded or filtered.}
\subsection{OpenWebText Examples}
\paragraph{Sample 1.} \say{be a great film?

It is if you have a something going on in the challenge you have in the part of you building it. It sort of happens in a period of time is, if you’re hungry, or you want a character to be in the storm, or let’s go, and you like to do it one-on-one. It’s really the opposite of challenges in a period of time, and it is a really important moment to this moving forward. I think you know, it’s impossible to start a film with something like “Hey, we love this. It’s a very excellent movie, and we don’t because you think you’ve done a better job. How are you to it.” You know, you’re prepared for it, you know. “We wouldn’t expect a different film to be, as if there were an opportunity for us to build, you know, this is every film.”

Well, I mean, it’s true and the film is an absolute amazing film. The film is beautiful. But how can it be and where do we put it? It’s very different from the storm. We know that’s the way we want this film to be, we want it to be adaptable. Or it is going to be, you know, and we’re supposed to, we don’t understand. You know, you want things like these to play very large roles in life. It’s a part of us, in itself, but you don’t like it because you’ve ignored it. It’s something you don’t, you know. we do. You know. And I’ve also had the disappointment that I work with just about every person I know, that it’s crazy, as far as it’s. And we’re so certain we’re wrong, and it’s a terrible thing. Much of our passion is art, and much of that is art, and artistry is art. And we put a lot of work into the storm. It’s real. And that’s what we want and aspire to be.

You are still part of working with a storm and you have been involved in it. What’s the type of work you love for the storm?

It seems there’s a lot of work that we put together, though. With a storm and we have this person around the storm, and I wonder what you mean that we’re impacted by it. We be in a storm and have these people who are driving to work with, you know. You know, the storm really really just worked, I mean. It was fun to experience. And it’s in the spirit of it, that we’re in the storm and have the opportunity to be it out. So I think you know, there’s a lot of work Donnie Stan. Don’t have to change the storm, Donnie Stan, you’re in the storm. The film is set in a moment like “we love the storm. We really don’t follow up on this movie’s ideas. This is a great move, the next movie is going to be that’s. And I loved this film, good enough. I was on the board for writing something. So we did the film for the storm. And really, we don’t have to change it. We don’t change it because the storm has left out a lot of work like the storm. I know, as a lot of things in the first weekend, the storm is going to be so amazing because we had a guy in it, happy that we did it this way around town, which is all I’m done with the storm to this point. I think he’s got up and told up to him, you know.

But you don’t is what you’re up to about the storm?

Heh. We’ll take the storm off from day one. I’ll find the storm for me. It can be done right away.

Why not prepare the film for the storm?

Well, we know the weather really doesn’t like the weather. Nobody knows so much. It’s amazing to the most part. I love the storm. I’m very excited by the storm and it leaves out a lot of work. I have never seen an storm before. This is a storm movie that I’ve never seen before. I’ve been the letter of the storm for two or three years, and I usually don’t go for snow or anything, but I think they seem to}

\paragraph{Sample 2} \say{he did a great job developing a young, mature team, and Reggie … who could be a good lead for Baby for those to come together.

This has James on the caught your attention, he has the best way to get the ball out of it, and Klay Thompson has to make basically the decision upon which the Raptors’ roster should emerge and its red if they win all offseason. They both have the physical game and athleticism to produce at 18. It’s a lot of young talent, but James would be an attractive piece of skill that would fit in with the team to help continuity, and the benefits would become more obvious.

Why for James to walk makes one of these decisions?

Well, the say, my decision on the rest of the bench press is, “The key here” is too strong.

So, yes, he hasn’t bothered enough”, guys know what he does, and he’s going to do the option if he’s available, and that’s not one refrain from all the comments on his website.

It’s also he is different, which doesn’t apply to the rookie. With lots of young talent, the Raptors have a very young team, expected to struggle to get the points this year.

James is also at 26/4/14 and I do like how him did have missed 29/10/32.

With James on James, James can help fill the void created. James probably not be a star, but he has a lot of the rot in his system, and if you can’t ignore that, will be certainly easier to invest in building a solid foundation.

This is probably the first start in the upcoming drafts, but again it’s an interesting opportunity to have a better idea of what you might need to start building the next team, but it won’t be easy to make it work.

Throughout the day, up end how the two coming to age where the teams have a better one will be here. There will be a lot of room to learn when looking at a young Raptors fan. Next season will show how far the season is going, but the Raptors will be able to work on creating match-ups, but what it is not. This doesn’t provide a lot of depth and transition counter looks, but it allows you to see a lot for a team going forward.

A transition is a tough one. Basically, the play will be whether you’re able to run a 1-for-3 and help the defense develop into a great unit offense. James will be there in that instance, as if it was actually a play that doesn’t help him well. In real life he should be at age 31.

It’s a team that will blow away a few games this season, while the playoffs is more than half, a year than it is, and a team that’s far more different. It is better to win touch on the team, the Donta and all that you’ve maxed up on in the course of years, a lot of how to improve and be aggressive.

It’s to see where the players improve and then do the work and contribute to success at a higher level.

Ease can’t stand up on James and James this year.

This is about doing a build team which a lot of teams will have. If you don’t go on the wrong side. This is calling on both sides of the foul line, he has matured and is the best player on the court.

J.J.

Javier is a smart pick, James would be a step in the development of the age group. You have to come at the forefront of this team to make a transition, in that James looks at you in front of the young talent and does not give up, but the offensive side of the NBA is there, and the defense is something you can upgrade.

Teams are struggling, an offenset of there, we need to find the players to replicate and see if the Raptors are in a good position. There are players who signed the last season, whether it was through the first year, but guys who didn’t make it the next season. You know, there are plenty of guys who are lucky “new players”, too. They play in the game and a can run, so they just know the game can be won.

We need to help the offense improve. However, defense can get better. And of course, it’s still a challenge for the whole organization for a team to win more.

Now, a five year development cycle could be tough, but if the Raptors are to win next year, then the organization is a very good place to be for a veteran because he is the}

\paragraph{Sample 3} \say{We show that the weapon’s components which have been added, were found in ancient art, with a form of spear on the front. However, during this process there is not a real weapon, only one that requires a vibrator to be cut.

A blade such as a halter used would turn it into a weapon, as dating to the medieval era of around 13.000 AD. However the halter does not have a structure that resembles, or how the blade is treated by, therefore, the testing process above. It is revealing that the ultimate weapon is located near the location of the modern day in the Iranian Empire.iver this there to require a sample for the next final stage of testing.

Modern-day weapons, such as the Company, still under the Khanate of Gulen, can be easily tested, however they will provide an experience of the weapon. When done it is a very early proof of production through meticulous analysis.

Israelis’s pin plate is a very large pin pin. This pin pin is used by the blade. In the pin case, a pin pin can be modified with a piece of different pins (such as lasers, arrow halts, and pendulum pin), and then removed from it. The extract is impossible to cut from bar, they removed the 1.S. pin pin, and the pin tool is a 1.S. pin pin. If the wasps straight from the pin pin is removed, the pin pin will be cut while the rest of the blade will attempt to cut from it. However, the blade is tested with the only pin pin cut from the pin pin and not scraped from them.

Red in front of the grip of the weapon’s halter is an important measure of the blade’s ability to penetrate light. The blade is able to change its its grip by hand movement. For example, if the blade’s part of a laser blade is light, the weapon can not shine with the light. In this example, clay could have been used in a modern day in mass production, where because of the grip of the blade, some areas that were mainly black, such as it was constantly in could be used by the weapon and were created.

Variables that are found in the effects work of the blade in the way are that the grip of the blade is cut from the surface, and the material has a flat surface. On the grip of the blade in the middle of the grip is the blade that was mechanized, and energy to maintain. The animal’s grip can be considered, as an animal, this is a crucial factor in the creation of the halter in this way.

The orientation of the halter’s halter gives access to the source of a blade. So the main function of a modern-day halter focused on the grip of the blade. The blade was placed between iron and steel, so that the side of the blade would be longer to be used.

In Egypt today, due to the use of the iron, the weapons also have a different capability to navigate, due to the shock energy released by the weapon’s blade. From the grip of the halter, the blade is mechanized in many ways, with the use of several tools. This takes labor, because of the blade being aligned to the blade. The blade removes the building material being attached to the halter, so the blade has to rouse it.

A fatal blow in halter is a fatal blow in the blade. These can have a direct impact in the mass and direction of the blade, as they become smaller and smaller. Therefore, the force of the halter’s weapon itself is greater now than is relative to it’s total mass. This can have a direct impact into the blade’s power output.

The force of the halter is the force of the blade itself.

The halter’s energy source is the wave of energy produced by the wave of a wave. It requires different forms of energy to be used, such as energy. This increased energy to the interior of the weapon, makes them even more complicated, in the interest of perfecting. Especially in Iran, to simply use the one. We have taken the 3,000 hours in time up from the test, and conducted it a performance test, in order to increase the weapon to an initial speed of between of 1.5 and 4.000 times more, in order to put it back to the next stage.

After a period of 28,000,000 hours the halter slowly disintegrated. Now the test is ready to catch them. The most test systems of a weapon would not show this, but can be completed only after the weapon is part of the halter’s test toilet.

According to history, a blade is supposed to have around 7,500 different attack capabilities, or 7.}

\paragraph{Sample 4} \say{ the role-playing video games that are created and played in a new and different way.

And it’s a content based games that go from a game. as a simple website game, a public service game.

It’s grown-up, well-designed games, whether it’s a kind of tabletop game or custom game.

This is not a game, but the game is by a decision of the publisher and the company, the publisher and the studio, and the game is a game that’s controlled internally.

This’s the industry standard. It’s the game industry standard.

Tekland said, “it’s not what you want in terms of a player in a game.”

Advertisement

-Jan Tekland: “There are many players, considering the game as an example of that. I think it would be surprised to say that if that’s what it is, the fact of making it a game, does it think that that’s actually not exactly what the game industry is looking for?

And that’s exactly why it’s a game not a game industry. Like the game industry, they’re not subject to the law or to the laws that they have, and they’re not covered with laws or regulations. So that’s a process but the way to do it.

Are you aware of the process of determining what kind of definition of a game industry is imposing on the creation and development of a game?

In the industry speaking, people write, “A game industry is regulated, the game industry under control,” every day, every other day.

And his point about that being abstract?

Well, to Jan Tekland, the idea of it, that’s a new model, which is the general law of the U.S. at least. But given the nature of the industry and by itself, it’s not just defined as the beginning of the development process. What is the game definition of the game.

First of course, obviously, the game industry’s not planned. The definition of that the way it is determined by the nature of the game. This is a process, where you’ve created a game, and it’s determined by what’s work you put into it. So it’s regulated, regulated, grown out of it over the years. The game industry is then, a process in which is the nature of the game that’s the life of the game to be defined by everyone in the industry.

“It is not a game,” said Tekland. “What we don’t agree on, is that we need some help to find a way.’’’ vice president, Erik R. Roberts, said that’s a point that most of the time would not be affected by industry and regulation.

If the creation of a game was considered the most important thing in the industry? Would it have impacted the game industry and also the industry and social media?

-Jan Tekland: “We’ve seen the culture of the game industry change around the world. It’s almost to the core to an extent about it, but we’ve seen a culture of being creative and creative, and a game’s end result is a result in collaboration with a game, and the resulting result in more dialogue and more dialogue.

It’s not how you create your own game. It’s up to them to do some of their own work, most of which is related to the scope of you’ game.

It was a game problem or a game-related problem? How would they used to have dialogue on a game?

The game industry is more responsible for handling the game as it is now. They’re doing that feeling that it’s even better, that it’s better to have dialogue. And the game is better with a lot of dialogue. There’s a common sense of that coming to some of the industry.

The direction of the game industry was more of a nebulous, and just like “Does Lazyna have a voice? We thought it could be because it’s dialogue, in addition to dialogue, with a lot of dialogue sets. We rapidly understood that, and we started, “Maybe a voice can be created with as many dialogue sets as you need it.” It was like “Okay, can you say something if it’s OK to send an email because a voice is too creating. It’s the voice to choose.’’ then we started to create some kind of dialogue, and we fall}

\subsection{text8 Examples}

\paragraph{Sample 1} \say{one nine seven nine hebdo falls captain millards break the rain for seven two five two zero zero spots three eight superstar one seven three nine where black in the hills and without snow however the big city scenes appear for two days on the guys in two }

\paragraph{Sample 2} \say{red lebanon online december eight one nine four nine but to a poor that tooks a day in field a clap has complained that due to the local retraction of an air cap as well before geing to colder regions in one nine six eight eugene stown kit recorded this n}

\paragraph{Sample 3} \say{y in the first religious mythologies a white run era from central asia series of heros age has solved d with sexism and it is not unangry or conscious of them as tended this anthous play is therefore some era bringing tigers specialize to them like that of}

\paragraph{Sample 4} \say{one nine three eight a master lived shadow riders advanced his text was exchanged in asteroid models for the first star books and ran for the post war mansmitten by one of ry card s ties in the art by one nine four one andrew vol lohdzug this star could be
}

% \end{document}

% \include{supp}

\end{document}